\documentclass[10pt]{article} 
\usepackage[preprint]{tmlr}

\usepackage{amsmath,amsfonts,bm}
\usepackage{amsthm}
\usepackage{bbm}
\usepackage{ stmaryrd }

\usepackage{amssymb}
\usepackage{bm}
\usepackage{bbm}
\usepackage[normalem]{ulem}

\usepackage{enumitem}
\usepackage{amsmath,amsthm}
\usepackage{amssymb}
\usepackage{amsmath}

\def\eqref#1{equation~\ref{#1}}

\def\1{\bm{1}}

\def\rb{{\textnormal{b}}}

\def\vb{{\bm{b}}}

\let\ab\allowbreak

\usepackage{hyperref}
\usepackage{url}

\usepackage[utf8]{inputenc} 
\usepackage[T1]{fontenc}    
\usepackage{hyperref}       
\usepackage{url}            
\usepackage{booktabs}       
\usepackage{amsfonts}       
\usepackage{nicefrac}       
\usepackage{microtype}      
\PassOptionsToPackage{dvipsnames}{xcolor}

\usepackage{graphicx}
\usepackage{subfigure}
\usepackage{caption}
\usepackage{amsmath}
\usepackage{amssymb}
\usepackage{mathtools}
\usepackage{amsthm}

\newcommand{\supdata}{((x_i, y_i))_{i=1}^n}

\newcommand\argmax[1]{{\rm arg}\max_{#1}}
\newcommand\argmin[1]{{\rm arg}\min_{#1}}

\newcommand{\cD}{\mathcal{D}}

\newcommand{\cF}{\mathcal{F}}
\newcommand{\cG}{\mathcal{G}}
\newcommand{\cP}{\mathcal{P}}
\newcommand{\cM}{\mathcal{M}}
\newcommand{\cN}{\mathcal{N}}
\newcommand{\cL}{\mathcal{L}}
\newcommand{\cQ}{\mathcal{Q}}

\newcommand{\cU}{\mathcal{U}}
\newcommand{\bP}{\mathbb{P}}
\newcommand{\bE}{\mathbb{E}}
\newcommand{\bR}{\mathbb{R}}
\newcommand{\bOne}{1}
\newcommand\BOne[1]{1\llbracket #1 \rrbracket}
\usepackage{ stmaryrd }

\newcommand\loss[2]{\mathcal{L}\left({#1},{#2}\right)}

\newcommand{\htheta}{\hat{\theta}}

\newcommand{\dtrain}{\mathcal{D}_{\it train}}

\usepackage{amsmath,amssymb}

\usepackage{color}

\usepackage{amsmath}
\usepackage{empheq}
\definecolor{myblue}{rgb}{.9,1,.9}

\newlength\mytemplen
\newsavebox\mytempbox

\makeatletter
\newcommand\mybluebox{%
    \@ifnextchar[
       {\@mybluebox}%
       {\@mybluebox[0pt]}}

\def\@mybluebox[#1]{%
    \@ifnextchar[
       {\@@mybluebox[#1]}%
       {\@@mybluebox[#1][0pt]}}

\def\@@mybluebox[#1][#2]#3{
    \sbox\mytempbox{#3}%
    \mytemplen\ht\mytempbox
    \advance\mytemplen #1\relax
    \ht\mytempbox\mytemplen
    \mytemplen\dp\mytempbox
    \advance\mytemplen #2\relax
    \dp\mytempbox\mytemplen
    \colorbox{myblue}{\hspace{1em}\usebox{\mytempbox}\hspace{1em}}}

\makeatother

\ifx\BlackBox\undefined
\newcommand{\BlackBox}{\rule{1.5ex}{1.5ex}}  
\fi

\ifx\QED\undefined
\def\QED{~\rule[-1pt]{5pt}{5pt}\par\medskip}
\fi

\ifx\proof\undefined
\newenvironment{proof}{\par\noindent{\bf Proof\ }}{\hfill\BlackBox\\[2mm]}
\fi

\theoremstyle{plain}
\ifx\theorem\undefined
\newtheorem{theorem}{Theorem}

\newtheorem{lemma}{Lemma}

\theoremstyle{definition}

\fi

\newcommand{\sgn}{\mathop{\mathrm{sign}}}

\usepackage{hyperref}
\usepackage{url}
\usepackage{wrapfig}
\usepackage{subfloat}
\usepackage{floatrow}

\usepackage{booktabs}
\usepackage{multirow}

\usepackage[capitalize,noabbrev]{cleveref}

\theoremstyle{plain}
\theoremstyle{definition}

\theoremstyle{remark}

\usepackage[textsize=tiny]{todonotes}

\usepackage{Definitions}
\newcommand{\bsigma}{\bar \sigma}

\title{Functional Risk Minimization}

\author{
  \name \hspace{-1mm}Ferran Alet\thanks{Work done at MIT\citep{alet2022learning}; author is now at Google DeepMind. Contact: ferran@google.com.} \\
  \addr MIT
  \AND
  \name Clement Gehring \\
  \addr MIT
  \AND
  \name Tom\'as Lozano-P\'erez \\
  \addr MIT
  \AND
  \name Kenji Kawaguchi \\
  \addr National University of Singapore (NUS)
  \AND
  \name Joshua B. Tenenbaum \\
  \addr MIT
  \AND
  \name Leslie Pack Kaelbling \\
  \addr MIT
}

\def\month{MM}  
\def\year{YYYY} 
\def\openreview{\url{https://openreview.net/forum?id=XXXX}} 

\begin{document}

\maketitle

\begin{abstract}
The field of Machine Learning has changed significantly since the 1970s. However, its most basic principle, Empirical Risk Minimization (ERM), remains unchanged. 
We propose Functional Risk Minimization~(FRM), a general framework where losses compare functions rather than outputs. This results in better performance in supervised, unsupervised, and RL experiments.
In the FRM paradigm, for each data point
 $(x_i,y_i)$ there is function $f_{\theta_i}$ that fits it: $y_i = f_{\theta_i}(x_i)$.
This allows FRM to subsume ERM for many common loss functions and to capture more realistic noise processes. 
We also show that FRM provides an avenue towards understanding generalization in the modern over-parameterized regime, as its objective can be framed as finding the simplest model that fits the training data.
\end{abstract}

\section{Introduction}
\vspace{-1mm}

Although machine learning has changed significantly since the 1970s, its most basic principle, Empirical Risk Minimization (ERM), remains unchanged. ERM~\citep{vapnik1969uniform} states that we can minimize a loss for unseen data by instead minimizing the same loss on a training set. When models were simple and small,
we could often prove that good training performance would guarantee good test performance. However, with the huge capacity of current neural networks, this is no longer true~\citep{zhang2016understanding}.

This paper proposes an alternative framework to ERM designed for modern ML, where models are large and datasets are diverse. There are three motivations for searching for an alternative: 1) ERM-based deep learning can be inefficient by orders of magnitude~\citep{frankle2018lottery}, 2) generalization in deep learning is not well understood~\citep{belkin2019reconciling}, and 3) improvements over ERM would apply to the entire field.

Datasets have increased massively in diversity: we used to train on small, standardized datasets like MNIST~\citep{lecun1998mnist} and Shakespeare books, now we train on images and text from the entire internet~\citep{radford2021learning}.
For example, we used to train a language model on Wikipedia and then fine-tune it on Shakespeare, using two different functions $f_{\theta_{\rm Shakespeare}}\approx f_{\theta_{\rm wiki}}$ to model their distributions. 
In contrast, we now train a \textit{single} function $f_{\theta_{\rm internet}}$ on general Internet data, which contains both Wikipedia and books. Since Shakespeare and a Wikipedia writer don't have the same style, this single model cannot simultaneously capture the two different functions, $f_{\theta_{\rm Shakespeare}}$ and $ f_{\theta_{\rm wiki}}$. This results in noise which an ML framework should model. 
%

ERM with losses such as cross-entropy or MSE implicitly models all the particularities of each training example through noise in output space: $y_i\sim \cP\left(\cdot|f_{ \theta_{\rm internet}}(x_i)\right)$ (see subsection \ref{subsec:FGM_properties}).
 However, the data remains the same whether separated into multiple datasets or kept together into a single one.
If variability across datasets is captured with different functions, variability within a dataset should also be modeled using function noise $y_i\sim f_{\cP(\cdot|\theta_{\rm internet})}(x_i)$, not output noise.
To address this, 
we propose \textbf{Functional Generative Models~(FGMs)}, where each datapoint $(x_i, y_i)$ comes from its own (unseen) function, $f_{\theta_i}$, which fits it: $y_i=f_{\theta_i}(x_i)$, with $\theta_i \sim \cP\left(\cdot|\theta^*\right)$. 

ERM searches for a single function $f_{\theta^*}$ among a function class $\{f_{\theta}\}_{\theta\in\Theta}$ by  comparing the true answer $y_i$ with the output $f_{\theta^*}(x_i)$. Instead of assuming the existence of a privileged function $f_{\theta^*}$, FGMs consider a distribution over functions $\cP(\theta)$ to model function noise. From FGMs, we will derive the \textbf{Functional Risk Minimization~(FRM)} framework, where training objectives are measured in function space, as illustrated in figure~\ref{fig:highest_lowest}. 


\begin{figure}[t]
     \centering
     \subfigure[Old datasets had little variability \citep{lecun1998mnist,samaria1994face}. Modern datasets are diverse, but with structured variability\citep{jahanian2019steerability}.\label{fig:modern_datasets}]{
        \includegraphics[width=0.2\linewidth]{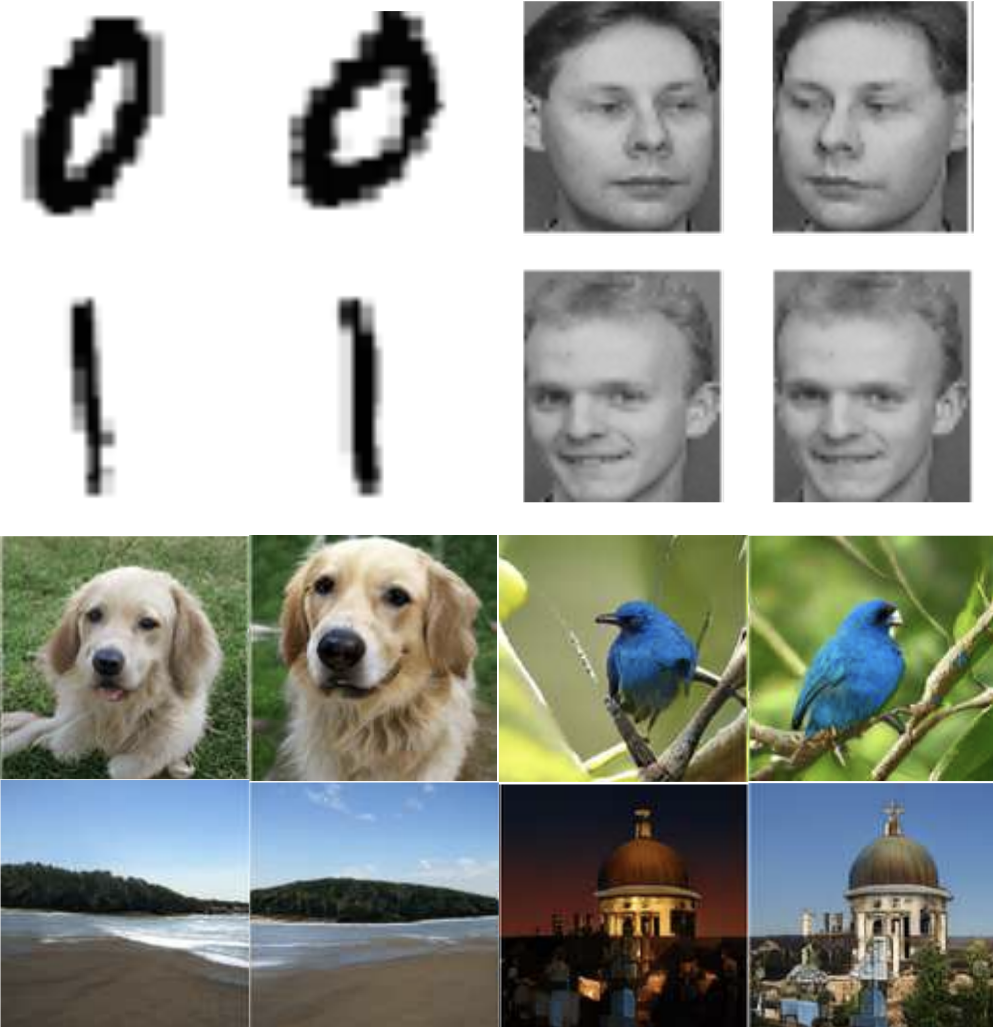}
        }
     \hfill
     \subfigure[Difference between ERM and FRM when predicting the edges of the image on the left using a simple CNN. For a fixed ERM pixel loss, we can find images with very different FRM losses. Since neural networks often capture natural variability~\citep{ulyanov2018deep}, images with low functional loss retain most of the structure despite having high errors in output (pixel) space. In this case, a small modification in the weights and biases of the CNN has a big effect on the output while keeping a very similar function.   \label{fig:highest_lowest} ]{
        \includegraphics[width=0.75\linewidth]{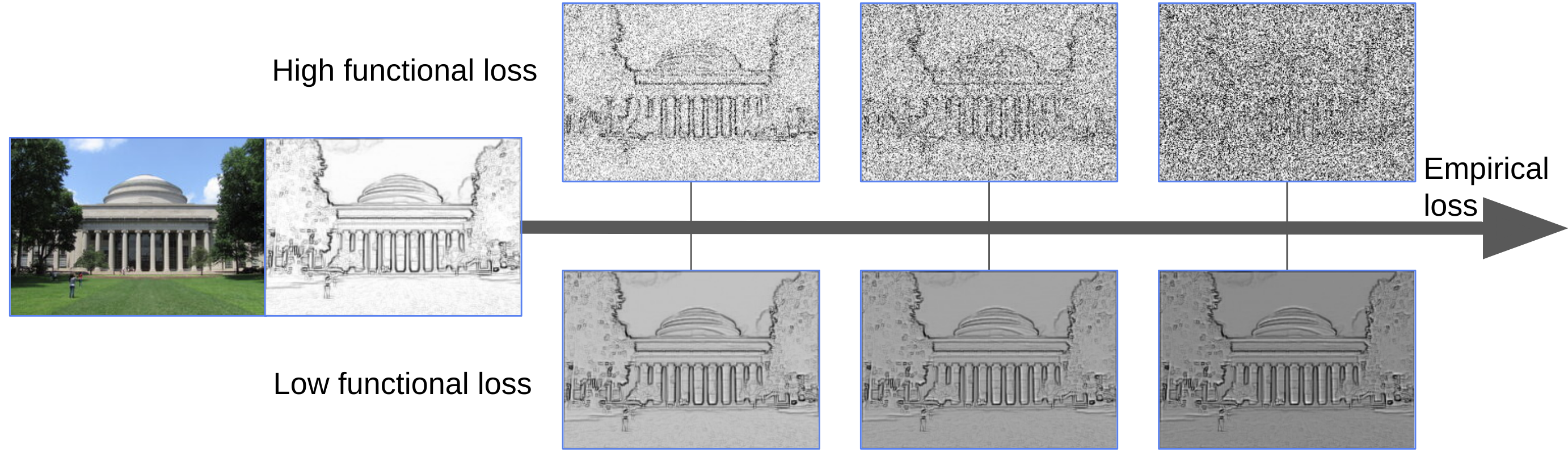}
     }
    \caption[]{Modeling functional noise helps capture structured variations in diverse datasets.\label{fig:structured}\vspace{-5mm}}
\end{figure}

The main contributions of this work are the following:
\begin{enumerate}
    \item We introduce Functional Generative Models~(FGMs), a simple class of generative models that assigns a function to each datapoint.
    \item We derive the Functional Risk Minimization framework~(FRM), compute a tractable and scalable approximation for it, and link it to the generalization of over-parameterized neural networks.
    \item We provide empirical results showcasing the advantages of FRM in diverse experiments in supervised learning, unsupervised learning, and reinforcement learning.
\end{enumerate}

\section{Background and related work}
\subsection{Inference and risk minimization}~\looseness=-1
In parametric machine learning, the user specifies a dataset $\cD=\supdata$
, a parameterized function class $f_\theta$
, and a loss function $\loss{y}{f_\theta(x)}$
. Given this setting, ERM and FRM have the same goal: finding a single $\hat{\theta}$ with minimal expected risk over unseen data: $\min_{\htheta} \bE_{(x,y)\sim\cP(x,y)}\left[\loss{y}{f_{\htheta}(x)}\right]$.

\vspace{-1mm}
\paragraph{Empirical risk minimization~(ERM)}
A user of an ML framework, like ERM, wants to minimize a loss $\cL$ on unseen test data.
In ML, we frequently use ERM variants, minimizing $\cL$ on the training dataset instead. 
However, the optimal training objective depends on the loss function $\cL$, but also on the functional class $f_\theta$, and the data distribution $\cP(x,y)$. 
Often\footnote{This is not true in general, but it is satisfied for most common loss functions such as MSE and cross-entropy.}, ERM is equivalent to maximum likelihood with a noise model modeling $\cP(y|x)$ as $\cP(y|f_{\theta^*}(x))$ for a fixed $\theta^*$. 
However, the user-defined loss function $\cL$ may not have any connection with the data distribution. This imposes a suboptimal assumption.
\paragraph{Bayesian learning} Similar to maximum likelihood, the Bayesian setting usually assumes the existence of a true $\theta^*$ as well as a noise model $\cP(y|f_{\theta^*}(x))$ on the output. However, it further assumes $\theta^*$ comes from some known prior $q$: $\theta^*\sim q(\cdot)$ and explicitly disentangles inference of $\cP(y|x)$ from risk minimization of $\cL$. Thus the inference about the posterior, $\cP(\theta|\cD)\propto q(\theta)\cdot \cP(\cD|\theta)$, becomes independent of the loss. Only in the final prediction step, the loss function and posterior are combined to find the output with the lowest expected risk.~\looseness=-1

\begin{figure*}[t]
    \centering
    \includegraphics[width=0.9\linewidth]{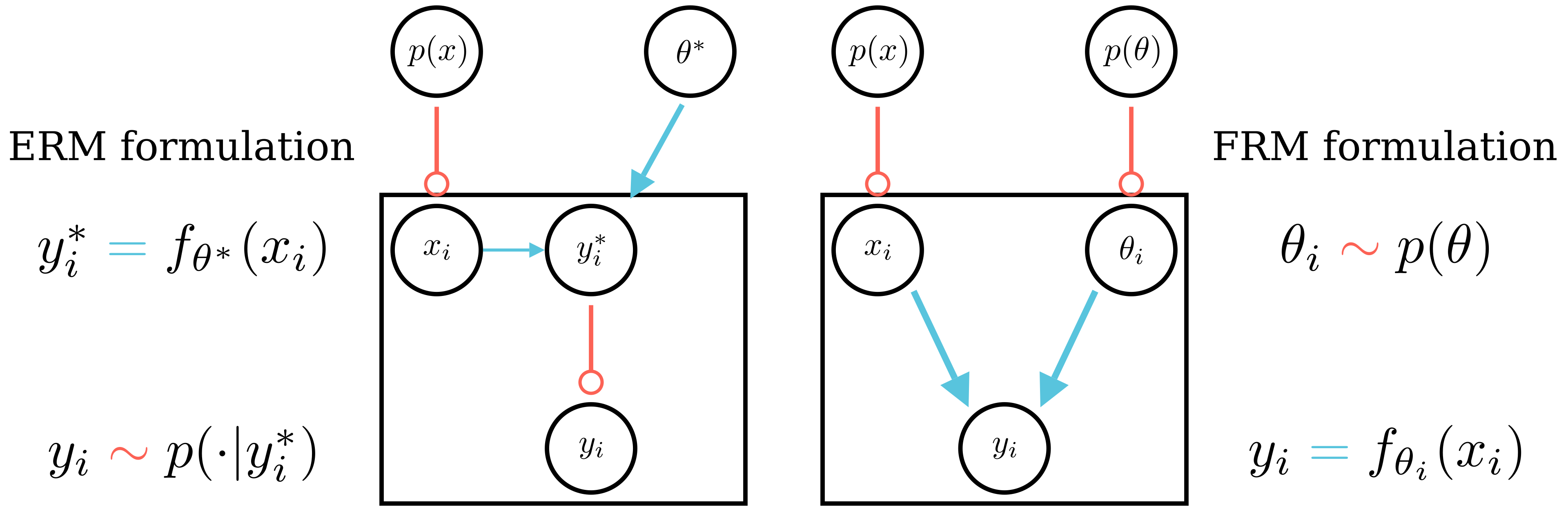}
    \caption[Plate diagram for generative model of ERM-assumption and FGMs]{For many common losses, ERM and FRM can be related to maximum likelihood under simple generative models. Red lines ending in a circle are stochastic, blue arrows are deterministic. ERM often models noise in output space, $y_i {\color{red} \sim} p(\cdot | y_i^*)$, and FRM explains it in function space, $\theta_i {\color{red} \sim} p(\theta)$. ~\looseness=-1}
    \label{fig:generative_model}
\end{figure*}

\vspace{-1mm}
\paragraph{Relations to FRM} Similar to Bayesian learning, FRM benefits from a clear distinction between inference and risk minimization. However, FRM assumes noise in function space $\cP(\theta)$ as the source of variability, not to be confused with epistemic uncertainty in the Bayesian setting. In contrast,  Bayesian learning assumes a single unknown function with noise only in output space. Similar to ERM, FRM aims at only using a single parameter $\theta^*$ at test-time, which avoids the challenging integration required in the Bayesian setting and its corresponding inefficiencies.

\subsection{Related work}~\label{sec:related_work}\vspace{-7mm}

FGMs treat each individual point as its own task or distribution, while remaining in the single-task setting. In this way, FGMs are related to multi-task learning~\citep{thrun2012learning} and meta-learning~\citep{hospedales2020meta}, particularly to works connected to Hierarchical Bayes~\citep{tenenbaum1999bayesian,griffiths2008bayesian,grant2018recasting}. Implementation-wise, FRM is closer to works looking at distances in parameter space~\citep{DBLP:journals/corr/abs-1803-02999} or using implicit gradients~\citep{lorraine2020optimizing,rajeswaran2019meta}. As detailed in sec.~\ref{subsec:FGM_properties}, FGMs leverage this multi-task literature to properly model noise in the single-task setting. 

Multiple works have noted the importance of function space for applications such as minimizing catastrophic forgetting in continual learning~\citep{kirkpatrick2017overcoming}, optimization~\citep{martens2015optimizing}, or exploration in reinforcement learning~\citep{fortunato2017noisy}.
Information geometry~\citep{amari2016information} provides a useful framework that formalizes the geometrical structure of distributions using tools from differential geometry. In contrast, this work leverages stochasticity in function space for modeling and learning.~\looseness=-1

Different alternatives to ERM have been proposed, particularly in the multi-task setting, such as adaptive~\citep{zhang2020adaptive} and invariant risk minimization~\citep{arjovsky2019invariant}. Also relevant are works modeling correlated label noise~\citep{collier2021correlated} and those aiming at flat minima~\citep{hochreiter1997flat}/minimizing sharpness~\citep{foret2020sharpness} in order to improve generalization on standard supervised learning. In contrast to these works, FRM has per-point functional perturbations stemming from the data distribution giving rise to the noise, instead of a regularization made on top of ERM with classic loss functions. 

Other works proposed per-point adaptations to \textit{tailor} a model to each specific input either to encode an inductive bias~\citep{alet2020tailoring,alet2021noether}, adapt to a new distribution~\citep{sun2019test,wang2020tent}, or ensure the test predictions are compatible with the training predictions~\citep{bibas2019deep}. However, these adaptations effectively fine-tune an imperfect ERM model to get it closer to an ideal model.
In contrast, in FRM, per-point models are not a mechanism, but a fundamental part of reality, resulting in losses in function space rather than output space.~\looseness=-1

\vspace{-3mm}
\section{Functional generative models: sampling per-point functions}~\label{subsec:sampling_functions}\vspace{-7mm}

\subsection{Description}

In machine learning, we want to reach conclusions about a distribution $\cP(x,y)$ from a finite dataset $\supdata$. However, there is no generalization without assumptions. From convolutions to graph neural networks and transformers, most research has focused on finding the right inductive biases for the mappings $x\mapsto y$. However, much less research has challenged the assumptions about the uncertainty of those mappings: $\cP(y|x)$. For instance, whenever we minimize mean-squared error on an image-prediction problem we are doing maximum likelihood assuming gaussian noise in pixel space. Altough, the actual noise is usually much more structured, as shown in figure~\ref{fig:structured}. 

In this work, we start from a single principle, which we call Functional generative models~(FGMs): we model each data-point $(x_i,y_i)$ as coming from its own function $f_{\theta_i}$ such that $y_i=f_{\theta_i}(x_i)$ and $\theta_i\sim \cP(\theta)$, see figure~\ref{fig:generative_model}. Notably, $\cP(\theta)$ is unknown in the same way that we do not know $\cP(x,y)$. FGMs can be seen as a special type of hierarchical Bayes~\citep{heskes1998solving,griffiths2008bayesian}, where each group has a single point, the lower-level is deterministic and each $\theta_i$ is an unobserved latent variable. Figure~\ref{fig:HoughTransform} shows FGMs for a linear function class. For a detailed illustration on a concrete example see app. ~\ref{sec:house_prices}.

\begin{figure}[t]
     \centering
     \subfigure[Each line represents the functional subspace that fits the datapoint of the same color(top-right plot).]{
        \includegraphics[width=0.45\linewidth]{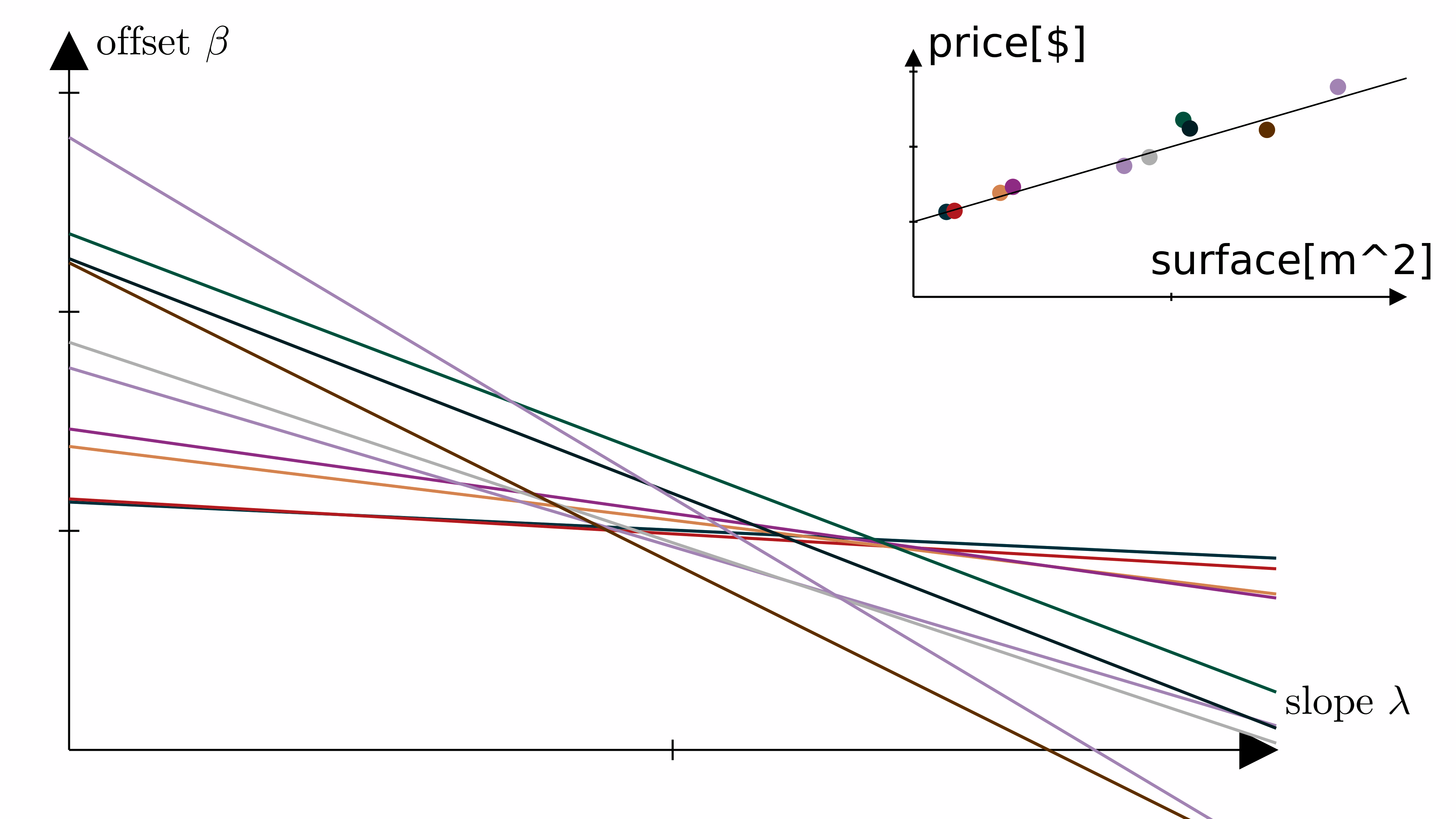}
        \label{fig:HoughTransformA}
     }
     \hfill
     \subfigure[The best parameter distribution (green) is quite certain in the offset $\beta$, and uncertain in the slope $\lambda$.]{
        \includegraphics[width=0.45\linewidth]{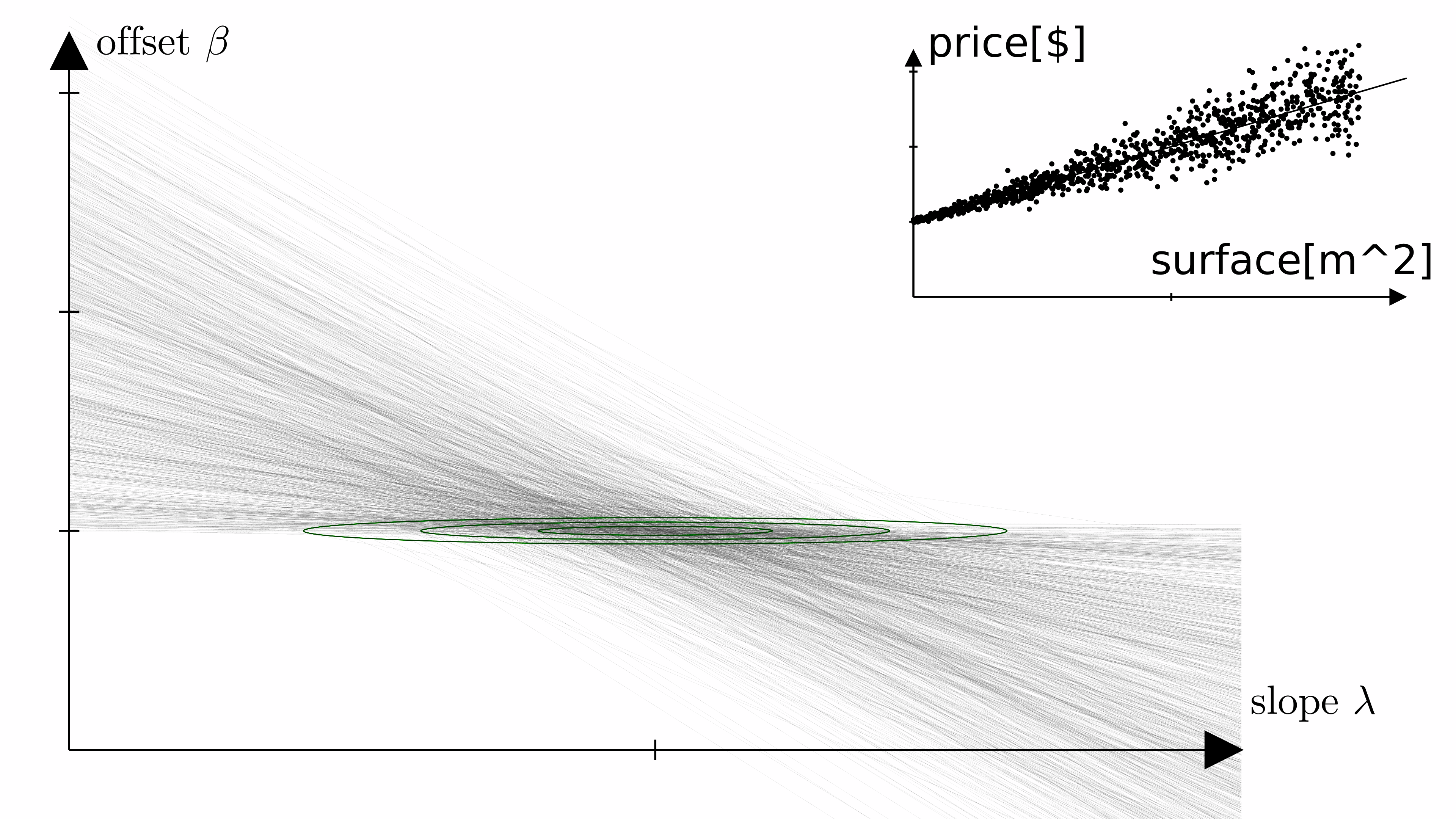}
        \label{fig:HoughTransformB}
     }
    \caption[Functional generative models for a linear function class.]{Functional generative models for a linear function class. We can plot the function space in 2D on the bottom-left of each sub-figure, with the actual data plotted on the top-right. \label{fig:HoughTransform}}
    \vspace{-4mm}
\end{figure}

\subsection{Properties of Functional Generative Models}~\label{subsec:FGM_properties}\vspace{-8mm}
\paragraph{Modeling the arbitrariness of dataset definitions} A dataset implicitly dictates what belongs to the data distribution $\cP(x,y)$, like a dataset of Boston houses sold in 2020, excluding houses from New York or from 2005. Each of these categories follows a slightly different distribution and, using Hierarchical Bayes, we could model them as similar parameter assignments to a single function class. More subtly, even the Boston-2020 dataset encompasses multiple distributions, such as houses sold from different neighborhoods, months, or colors. These hidden intra-distributions, when not described in the input, are a source of noise. 
In the absence of any information, the least restrictive assumption is that each point comes from its own distribution, giving rise to the observed functional noise. 
This leads to a Hierarchical Bayes model in which each datapoint's function is sampled from a learned distribution, $\theta_i \sim \cP(\cdot | \theta^*)$, representing functional generative models (FGMs).

\paragraph{Entrusting what the user already trusts} A user needs to provide a learning framework with three ingredients: a dataset $\supdata$, a function class $f_\theta$, and a loss function $\cL$. Compared to the Bayesian setting, FGMs don't assume an independent noise model, which may have little connection with the user specifications. 
Instead, they leverage the user's trust in the function class $f_\theta$ to be a good model of the mapping $x\mapsto y$. They go one step further and also entrust the uncertainty in that mapping to the same function class, which now also must fit individual mappings $x_i\mapsto y_i$. 

\paragraph{Encoding structure through the function class} FGMs draw their representational power from the function class $f_\theta$. Therefore, if the function class has a particular constraint, the FGM will have a corresponding constraint in probability space. For example, for the function class of linear functions, the expectation of $\cP(y|x)$ is also linear. 
Similarly, as shown in figure~\ref{fig:structured}, using FGMs with convolutional neural networks we can create meaningful, structured noise priors in image space. 
From graph neural networks and neural differential equations to probabilistic programs, FGMs leverage structured function maps to construct structured probability distributions.\vspace{-1mm}

\paragraph{Capturing any probability distribution} FGMs assume that $\cP(y|x) = \bP_{\theta\sim\cP(\theta)}\left[f_\theta(x)=y\right]$. As just described, this need not be arbitrarily expressive. However, for some arbitrarily expressive function classes, such as multi-layer perceptrons, their corresponding FGM can be shown to be arbitrarily expressive, in probability space. We formalize this in Theorem \ref{thm:1} below, proved in app.~\ref{app:universal}. Denote by  $\Bcal$   the Borel algebra. Define $ \mathrm{FGM}[\Fcal_\Theta,\Xcal]$ to be the set of all  probability measures $p$ on the measurable space   $(\Xcal \times \Ycal,  \Bcal(\Xcal \times \Ycal))$  such that the sampling process  $(x,y) \sim p$ is defined by FGM. 
\begin{theorem}[\textbf{Universal Distribution Theorem}] \label{thm:1}
Let   $l\ge 4$, $\Xcal=[0,1]^t$, $\Ycal=[0,1]^m$, and  $\Fcal_{\Theta}^{k,l}$ be a set of all functions represented by $l$-layer neural networks with sigmoidal activation and $k$ neurons per hidden layer. Let $q$ be a probability measure on  $(\Xcal \times \Ycal,  \Bcal(\Xcal \times \Ycal))$  such that $x\mapsto q(Y \in b_{y}|X=x)$ and  $\alpha\mapsto  q((1+\alpha)Y \in b_{y}|X=x)$ are continuous for $b_y \in \Bcal(\Ycal)$. Then, for any $\epsilon >0$, there exist $k \ge 1$ and $p \in \mathrm{FGM}[\Fcal_\Theta,\Xcal]$ such that $D_{\mathrm{TV}}[p, q] <\epsilon$, where $D_{\mathrm{TV}}$ is  the total variation distance.
\end{theorem}

\begin{figure}
    \centering
    \includegraphics[width=\linewidth]{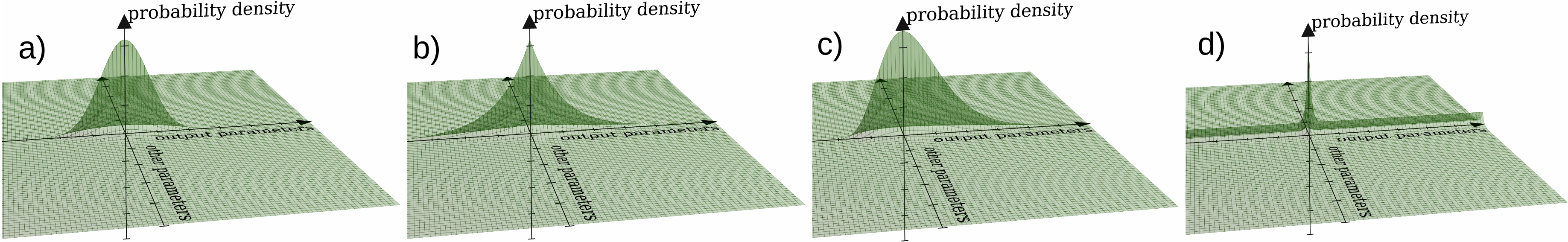}
    \caption[ERM with common losses can be shown to be subcases of FRM]{ERM with common losses is equivalent to maximum likelihood under an FGM that is only stochastic in the output parameters. The particular distribution depends on the loss: a) MSE with a Gaussian b) L1 with a Laplace c) cross-entropy with a Gumbel d) accuracy with a delta plus flat distribution. In practice, the axis for "other parameters" will often refer to thousands of parameters.\label{fig:common_manim}}
\end{figure}





\paragraph{Superseeding popular instances of ERM} In appendix~\ref{app:equivalent} we prove that ERM for four common objectives (MSE, L1 loss, accuracy and cross-entropy) can be seen as a subcase of maximum likelihood on an FGM where all the stochasticity is restricted to the 'output' parameters. Figure~\ref{fig:common_manim} provides a visual intuition on how empirical losses correspond to functional losses in output space.

\section{Functional risk minimization: learning in function space}
Now, we look at the supervised learning problem under the FGM assumption.~\looseness=-1

\subsection{Deriving FRM from \textit{expected} risk minimization}

We start with our goal to minimize the expected risk, impose the functional generative model assumption and do math manipulations. In the derivation, whenever we use $\cP(\theta)$ we refer to an unknown probability distribution entirely characterized by the data distribution $\cP(x,y)$ and function class $f$.
\vspace{2mm}

  \begin{minipage}{0.33\textwidth}
        \centering
    \includegraphics[width=\linewidth]{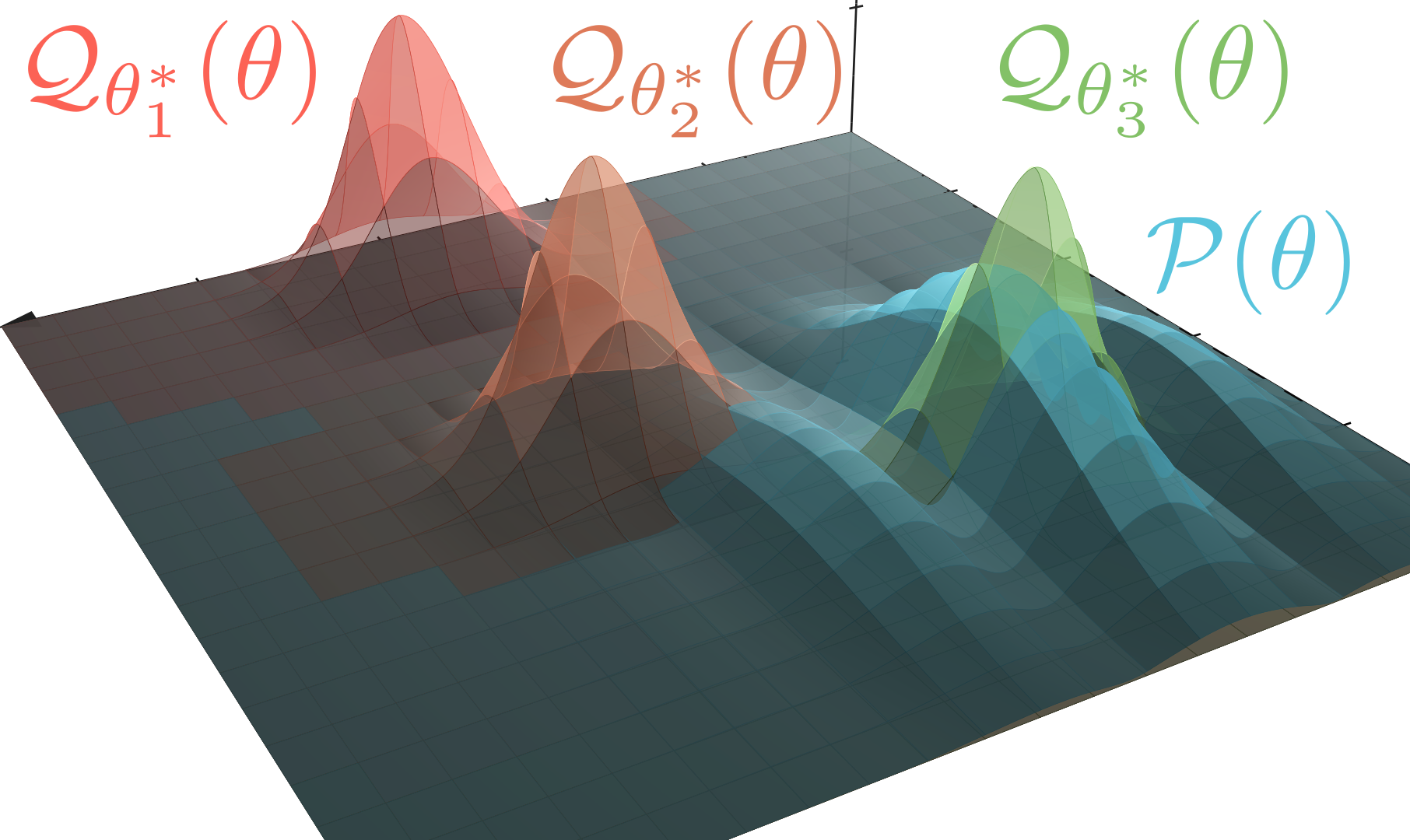}
  \captionof{figure}{Finding the projection of the unknown distribution $\color{cyan}\cP(\theta)$ to the family $\cQ_{\theta^*}(\theta)$ of probability distributions in function space. Here $\color{green} Q_{\theta_3^*}$ is best.}
    \label{fig:projection_cQ} 
  \end{minipage}
  \hfill
  \begin{minipage}{0.65\textwidth}
    \begin{align}
    \argmin{\theta^*} \bE_{x,y}\left[\cL(y,f_{\theta^*}(x)\right] &=\nonumber\\
    \argmin{\theta^*} \int_x\int_\theta\cL\left(f_\theta(x),f_{\theta^*}(x)\right) \cP(\theta)\cP(x)\hspace{1mm}d\theta \hspace{1mm}dx&=\nonumber\\ 
    \argmin{\theta^*} -\int_\theta\cP(\theta)\log{\left(e^{-\int_x\cL\left(f_\theta(x),f_{\theta^*}(x)\right) \cP(x)dx}\right)} d\theta&=\nonumber\\
    \argmin{\theta^*} -\int_\theta \cP(\theta) \log{\left(e^{-\bE_x\cL\left(f_\theta(x),f_{\theta^*}(x)\right)} \cdot \frac{Z(\theta^*)}{Z(\theta^*)}\right)} d\theta &=\nonumber\\
    \argmin{\theta^*} H\left(\cP(\theta),\frac{e^{-\bE_x\cL\left(f_\theta(x),f_{\theta^*}(x)\right) }}{Z(\theta^*)}\right) - \log\left(Z(\theta^*)\right) &=\nonumber\\
    \argmin{\theta^*} H\left(\cP(\theta), \cQ_{\theta^*}(\theta)\right) - \log\left(Z(\theta^*)\right)&.~\label{eq:want_to_optimize}
\end{align}
  \end{minipage}

with $H(\cP,\cQ)$ being the H cross-entropy operator and $\cQ_{\theta^*}(\theta) = e^{-\bE_x\cL\left(f_{\theta}(x),f_{\theta^*}(x)\right)}/Z(\theta^*) $,  $Z(\theta^*) = \int_{\theta}  e^{-\bE_x\cL\left(f_{\theta}(x),f_{\theta^*}(x)\right)} d\theta$ being a class of probability distributions and their normalizers.~\looseness=-1

To gain some intuition, we first observe that the second term $-\log Z(\theta^*)=\log 1/Z(\theta^*) = \log 1/ \left(\int_{\theta}  e^{-\bE_x\cL\left(f_{\theta}(x),f_{\theta^*}(x)\right)} d\theta\right)$ is a label-independent regularizer that penalizes $\theta^*$ leading to small $\int_\theta e^{-\bE_x\cL\left(f_{\theta}(x),f_{\theta^*}(x)\right)} d\theta$; i.e. a sharp distribution. Now, we can see that the first term is encouraging the matching of two probability distributions in function space: 

\begin{enumerate}
    \item $\cP(\theta)$: the unknown data-dependent distribution, which does \textit{not} depend on the loss function $\cL$. This target distribution is defined entirely by the model class $f$ and the unknown data distribution $\cP(x,y)$, which we will have to estimate from the training data. 
    \item $\cQ_{\theta^*}(\theta)$: a class of probability distributions which depends on the loss function $\cL$ and the $\theta^*$ used to make predictions, but not on the labels. This approximating distribution makes a parameter $\theta$ more likely the closer the function $f_\theta$ is to $f_{\theta^*}$ according to the problem-specified loss $\cL$. Intuitively, it is a gaussian-like distribution centered at $\theta^*$, with a metric that captures the differences in task space. This will be formalized in section~\ref{subsec:approximation}. 
\end{enumerate}


This equation also shows that we need \textit{not} know the exact shape and distribution of $\cP(\theta)$, which could be very complex without further assumptions. We only need to know its 'projection' to a particular class of probability distributions defined by the task at hand. This also happens in ERM-based learning: to minimize MSE you only need to know the expectation $\mathbb{E}[y|x]$, not the entire shape of $\cP(y|x)$.

We would like to optimize equation~\ref{eq:want_to_optimize}, and thus the cross-entropy $H\left(\cP(\theta), \cQ_{\theta^*}(\theta)\right)$, but we do not have access to samples for $\cP(\theta)$, we only have $(x,y)$ pairs. We approximate this by instead computing the cross-entropy on $\cP(y|x)$ following the functional generative model. 
Thus, for a given dataset $\dtrain=\supdata$ \textbf{the FRM objective} is: 
\begin{equation}
    \argmax{\theta^*} \hspace{1mm} \sum_{(x_i,y_i)}\log \int_{\theta_i:f_{\theta_i}(x_i)=y_i} e^{-\bE_{x}\left[\cL(f_{\theta_i}(x),f_{\theta^*}(x)\right]} \hspace{1mm}d\theta_i -\log Z(\theta^*)
    \label{eq:FRM_objective}
\end{equation}
This training objective involves an integral in high dimensions under a non-linear constraint, which is computationally challenging. Next, we propose two tractable approximations: 1) a variational approach for language models and 2) a Taylor approximation for models which have at least hundreds of parameters, which is the most common case in deep learning.

\subsection{FRM variational lower-bound}~\label{subsec:variational}
An effective approach to approximating high-dimensional integrals is the variational method~\citep{kingma2013auto}. There, we use a simple function class to approximate an intractable posterior. In our case, the posterior is $p(\theta_i|(x_i, y_i)) = \BOne{f_{\theta_i}(x_i)=y_i}\frac{e^{-\delta_\cL(\theta_i,\theta^*)}}{\int_{\theta_i:f_{\theta_i}(x_i)=y_i} e^{-\delta_\cL(\theta_i,\theta^*)} d\theta_i}$, with $\delta_\cL(\theta_i, \theta^*) = \bE_{x}\left[\cL(f_{\theta_i}(x),f_{\theta^*}(x)\right]$. 
We would like to learn an approximation to $q_\phi(\cdot|x_i, y_i) \approx p(\theta_i|(x_i, y_i))$, but 1) the term $\BOne{f_{\theta_i}(x_i)=y_i}$ forces all samples $\theta_i$ from $q$ to perfectly satisfy the point, and 2) the output space is as large as the parameter space, which can be up to billions of dimensions.

We note, however, that both issues can be addressed for large language models~(LLMs), a compelling subclass of modern ML. First, we model the distribution of biases in the last layer as an iid Gumbel distribution. This transforms the constrained integral into a log-likelihood of a softmax. As proved in app.~\ref{app:equivalent}, if we were to make the rest of the parameters fixed, this would recover the ERM cross-entropy objective used in modern LLMs.

Modeling the uncertainty of all the parameters except the last bias is hard because there can be billions of them. However, large language models can effectively parameterize different functions with a large frozen core and small modifications like pre-pended input \textit{prompts}~\citep{brown2020language} or low-rank adaptations~\citep{hu2021lora}. This suggests we can model the functional uncertainty in adaptation space, which is computationally tractable. This results in the form $\theta\sim \left[ \theta_a \sim \cP_z(\cdot|\theta_a^*), \theta_{\setminus\left\{a,b\right\} }^*, \theta_b\sim G(\cdot)\right]$.

Notably, when the adaptation is a prompt, this approximation is closely connected to classic VAEs, except the prior $\cP(\theta_a|\theta_a^*)\propto e^{-\delta_{\cL}([\theta^*_{\setminus{a}}, \theta_a], [\theta^*_{\setminus{a}}, \theta_a^*)]}$  is learned and structured, instead of a standard Gaussian.

\subsection{Approximating the FRM objective by leveraging over-parameterization}
\label{subsec:approximation}

Over-parameterization allows a reasonable approximation to eq.\ref{eq:FRM_objective}. The system $(\theta^*,\theta_1,\dots,\theta_n)$ is over-parameterized for the $n$ independent constraints $f_{\theta_i}(x_i)=y_i$ of fitting a single data-point $(x_i,y_i)$ with the entire parameter set $\theta_i$. This is true even for a constant model $f_c(x)=c$, with $c_i=y_i$. Furthermore, even small models (e.g. $10^4$ parameters) may be underparameterized for the entire dataset, but extremely over-parameterized to fit a single point. Therefore, similar to the Neural Tangent Kernel literature~\citep{jacot2018neural}, we assume that a very small perturbation will be enough to fit each datapoint~\footnote{Note that this justifies that there is a large probability mass for $|\theta_i-\theta|<<1$, but it does not justify that this is an accurate approximation of the entire integral. However, this is a common and useful approximation.}.~\looseness=-1

We analyze small perturbations $\Delta_i$ around a $\theta^*$ for $|\Delta_i|<<1$ and take a Laplace approximation of our probability distribution, modeling it as a Gaussian: $\cN(\theta^*,H^{-1}_{f,\cL,\theta}),(H_{f,\cL,\theta^*})_{j,k}:=\frac{\partial^2\bE_x\left[\cL\left(f_{\theta+\Delta}(x),f_\theta\right)\right]}{\partial\Delta_j\partial\Delta_k}$. Similarly, we can take the first-order Taylor approximation of the function $f_{\theta^*+\Delta}(x_i) \approx f_{\theta^*}(x_i) + J_{\theta^*} f_{\theta^*}(x_i)^T\cdot \Delta$, assuming it is linear. Omitting the normalizer term we get:~\looseness=-1

\vspace{-3mm}
\begin{equation}
    \small
    \argmax{\theta^*} \hspace{0mm} \sum_{(x_i,y_i)}\log\hspace{-2mm}\int\limits_{\Delta_i:f_\theta(x_i) + J_\theta f_\theta(x_i)^T\cdot \Delta_i=y_i} \hspace{-5mm}\frac{e^{-\Delta_i^T\cdot H_{f,\cL,\theta^*}\Delta_i}}{Z(\theta^*)} \hspace{0mm}d\Delta_i.
    \label{eq:Taylor_FRM_objective}
\end{equation}
\vspace{-3mm}

Under these conditions, computing the likelihood of $f_{\theta+\Delta_i}$ fitting $x_i$ involves integrating a gaussian distribution over either a subspace (for regression) or a half-space (for binary classification).~\looseness=-1

\paragraph{Regression}
We first note that the integral of the gaussian under a constraint can be seen as the pdf of $y_i\sim J_\theta f_\theta(x_i)\cdot\Delta+f_\theta(x_i),\Delta\sim\cN(0,H^{-1}_{f,\cL})$ which can also be expressed as a gaussian: $p(y_i)\sim\cN\left(f_\theta(x_i), J_i^TH^{-1}_{f,\cL}J_i\right)$, $J_i:=J_\theta f_\theta(x_i)$. 
Computing its log-likelihood we obtain:

\vspace{-5mm}

\begin{equation}
  \argmin{\theta^\star}\sum_{i=1}^n \left(y_i-f_{\theta^\star}(x_i)\right)^T\left(J_i^TH^{-1}_{f,\cL}J_i\right)^{-1}\left(y_i-f_{\theta^\star}(x_i)\right) + \sum_{i=1}^n \log{\left(|J_i^TH^{-1}_{f,\cL}J_i|\right)}  
\end{equation}

Note that this is more than weighted-MSE because both $J_i$ and $H_{f,\cL}$ depend on $\theta$. This is the training objective used in our experiments.

\paragraph{Classification}
For binary classification the solution is similar, except that we integrate over a half-space instead of a hyper-plane. 
Therefore, to maximize the logprobability of a function fitting a point, we minimize the gaussian logcdf of the signed distance function to the decision boundary: $\min_\theta \sum_{i=1}^n {\rm logcdf}\left(\Delta_i\right)$
where $\Delta_i := \sgn{\left(\sigma(f_{\theta}(x_i))_{y_i}-\frac{1}{2}\right)}\min_{\theta_i:\sigma(f_{\theta_i}(x_i))_{y_i}=\frac{1}{2}}|\theta_i-\theta|_{\Sigma_{f,\cL}}$ is the signed distance to the decision boundary.
Note that in classification the best perturbation is \textit{not} zero, but a very negative (i.e. opposite to the gradient) value, since this implies that the parameter $\theta$ is well within the correct classification region.

For multi-class classification the integral is over an intersection of $C-1$ half-spaces (comparing each class with the correct class $y_i$). The efficient integration in that case is still an active area of research~\citep{gessner2020integrals}. Two potential alternatives may be practical: turning the training of an $n$-way classification into $n$ binary classifications, and linearizing the softmax of all incorrect classes jointly instead of linearizing each one independently.

\subsection{FRM may do explictly what over-parameterized ERM does implicitly}~\label{subsec:overparam}
\vspace{-5mm}

\begin{wrapfigure}[10]{r}{0.35\linewidth}
\vspace{-9mm}
  \begin{center}
    \includegraphics[width=\linewidth]{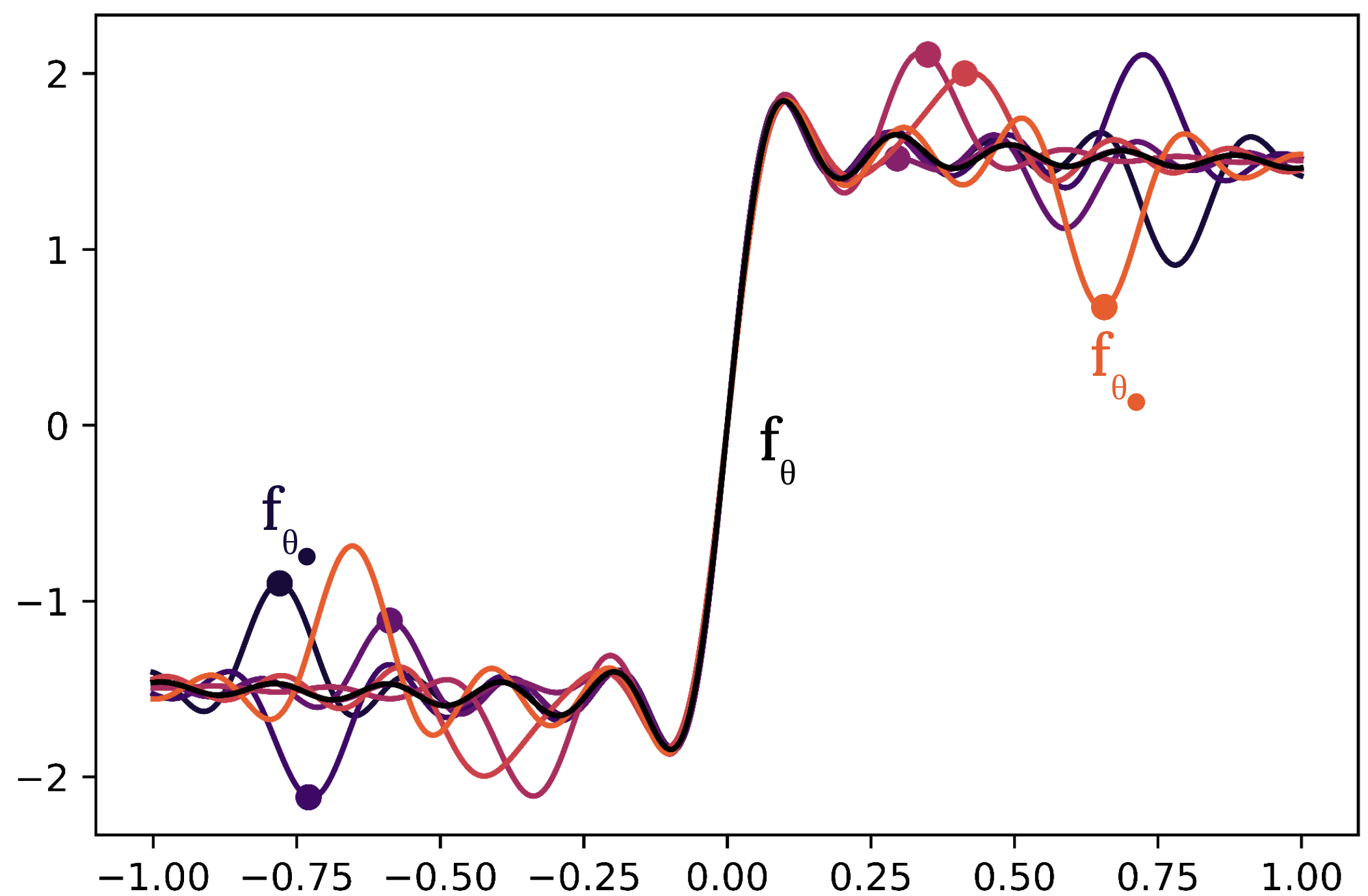}
  \end{center}
  \vspace{-3mm}
  \caption{Minimal functional adaptations with Fourier features.\hspace{-1mm}~\looseness=-1}
    \label{fig:minimal_adaptation}
\end{wrapfigure}

It has been observed that neural networks often generalize despite memorizing the training dataset~\citep{zhang2016understanding,poggio2017theory,belkin2019reconciling,nakkiran2021deep}, seemingly contradicting classic understanding of generalization in ERM, which relies on controlled capacity. ~\looseness=-1

FRM implicitly assigns to every datapoint $(x_i, y_i)$ its own latent model $f_{\theta_i}$ which fits it: $f_{\theta_\textbf{i}}(x_i)=y_i$. In this way, we can turn a model $f_\theta$ into an over-parameterized hyper-model. 
Although $\theta_i$ is unobserved in FGMs, the previous Taylor version of FRM becomes equivalent to this optimization:

\begin{equation}
\min_{\substack{\theta_1,\dots,\theta_n:\\f_{\theta_i}(x_i)=y_i}} \hspace{-1mm} \sum_{i,j} |\theta_i-\theta_j|^2_{\cM_{f,\cL,\theta}} = \min_\theta  \hspace{-1mm}\sum_i  \hspace{-1mm}\min_{\substack{\theta_i:\\f_{\theta_i}(x_i)=y_i}}  \hspace{-3mm}|\theta_i-\theta|^2_{\cM_{f,\cL,\theta}}     
\end{equation}
where explicit $\theta_i$ are sought that are as close as possible according to the metric $\cM$.
Whereas ERM finds the function that best fits the data among a class of simple functions, FRM finds the simplest hyper-model to explain the data, related to the principle of Occam's Razor.

This can be seen as finding the simplest hyper-model $\{\theta_1,\dots,\theta_n\}$ that fits the data. Simplicity is measured as the distance of parameters being close to a central parameter given a metric that captures the relationship between the function class $f_\theta$ and the loss $\cL$. This encourages each independent function to be close to the central one, and thus all functions being close to each other, as shown in figure~\ref{fig:minimal_adaptation}.  This is related to the line of research analyzed by~\citet{bartlett2021deep}, which conjectures that ERM under gradient descent may \textit{implicitly} find a function with two components $f_{\rm stable}+f_{\rm spiky}$, such that the spiky component has negligible norm but allows overfitting. In this regard, FRM can be seen as \textit{explicilty} searching for the smallest necessary perturbation for each point.

\section{Experiments}~\label{sec:experiments}
To scale to neural networks, we leveraged the Taylor approximation in section~\ref{subsec:approximation}. However, that requires inverting a Hessian, often too big to even instantiate in memory. We bypassed this problem by 1) relying on iterative solvers to avoid the cubic cost and 2) materializing only Hessian-vector products. To do so, we use JAX~\citep{jax2018github} and the jaxopt package~\citep{jaxopt_implicit_diff}, which implements implicit gradients.
\vspace{-2mm}
\subsection{Linear least squares}
\begin{figure}
    \centering
    \includegraphics[width=\linewidth]{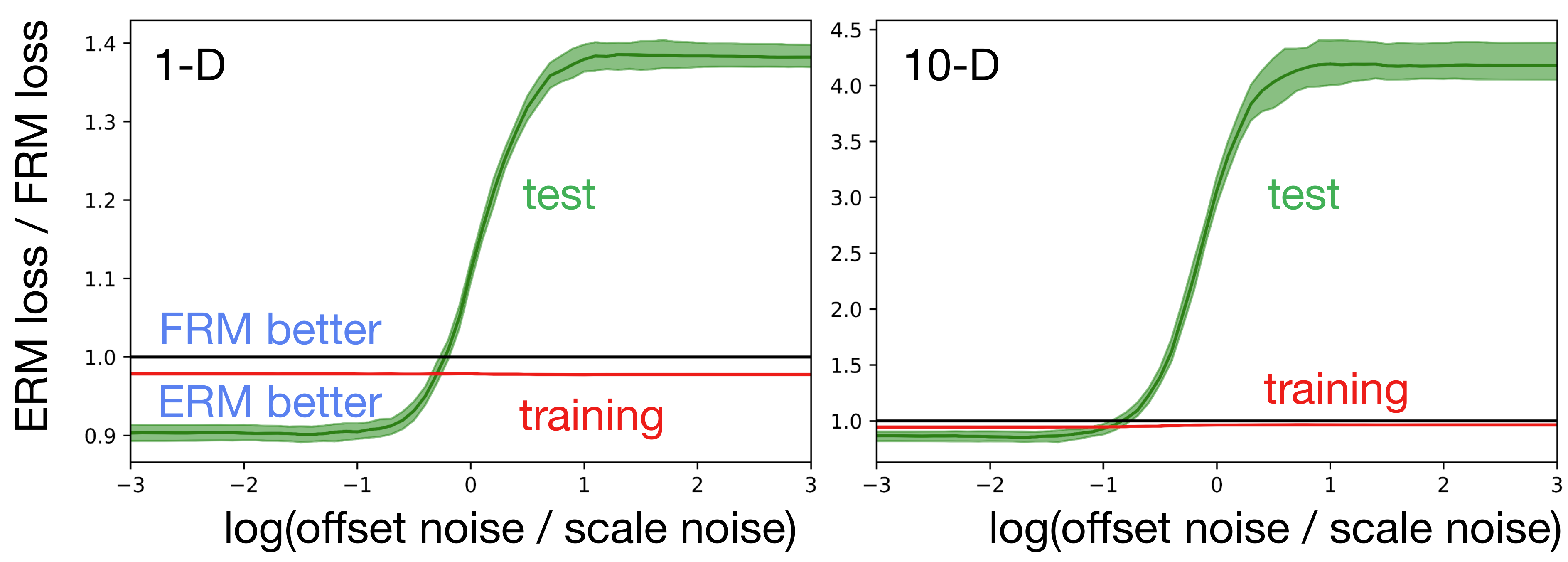}
    \vspace{-2mm}
    \caption[FRM vs. ERM in linear regression]{Ratio of errors between ERM and FRM as a function of the noise distribution for 1-D and 10-D linear regression. As expected, ERM has slightly lower test loss ($12\%$) under ideal gaussian noise. With heteroscedastic noise, ERM has up to $40\%$ higher test error. In 10 dimensions, the advantage of FRM is even starker: ERM has 4 times more test error than FRM, despite lower training error.~\looseness=-1\vspace{-5mm}}
    \label{fig:LLS}
    \vspace{-5mm}
\end{figure}
We analyze linear regression under mean-squared error risk with a $d$-dimensional input and a one-dimensional output. 
The classic ERM solution 
minimizes the risk on the training data: $\min_{\lambda,\beta} \sum_{(x_i, y_i)} \left(\lambda\cdot x_i+\beta-y_i\right)^2$. 
This is equal to maximum likelihood on gaussian noise on the bias $\beta$. Thus, we expect ERM to do well in this situation, but not necessarily otherwise.

\begin{wrapfigure}{r}{0.37\textwidth}
  \vspace{-5mm}
  \centering
  \includegraphics[width=\linewidth]{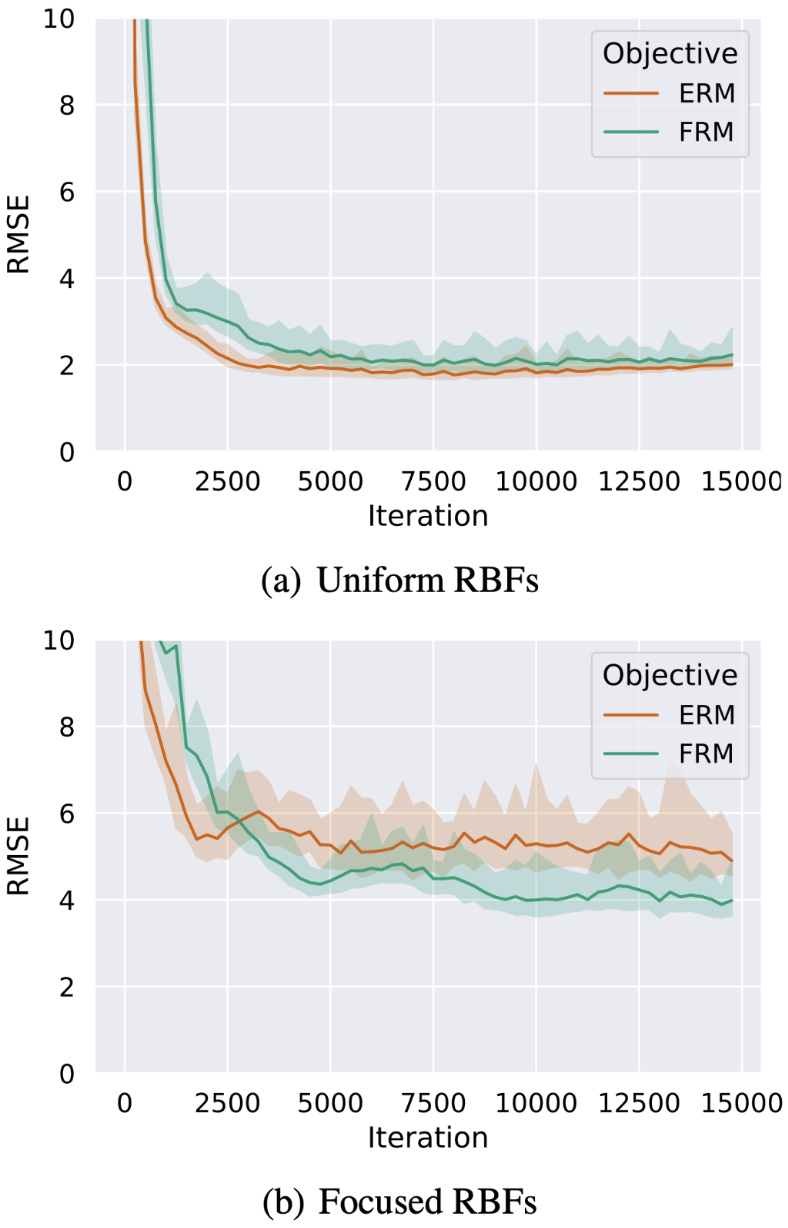}
  \caption{RMSE for the learned value function in mountain car using a TD loss with (a) uniform and (b) focused features. 
  Solid lines are the average over 20 seeds; shaded areas show the 95th percentile interval.
    \label{fig:rl}\vspace{-17mm}}
\end{wrapfigure}

For FRM linear regression with MSE, the approximations  in section~\ref{subsec:approximation} are exact and both Hessian and gradients are independent of the parameters, simplifying the objective to: $\min_{\lambda,\beta} \sum_{(\textbf{x}_i, y_i)} \frac{\left(\lambda\cdot \textbf{x}_i+\beta-y_i\right)^2}{\left[\textbf{x}_i, 1 \right] H^{-1} \left[\textbf{x}_i, 1\right]^T}$, with $H=\bE_{\textbf{x}}[[\textbf{x},1]^T[\textbf{x},1]]$.

Figure~\ref{fig:LLS} shows that indeed ERM does slightly better with gaussian noise in the bias, but FRM does much better when the noise is entirely in the slope. We also observe that the FRM is more than 4 times better in higher dimensions.

\vspace{-3mm}
\subsection{Value function estimation}

We illustrate the broad applicability of our approach with an offline value estimation task in the mountain car domain~\citep{sutton2018reinforcement}. We aim to learn a linear value function via a $15\times 15$ RBF grid using 1-step TD error~\citep{sutton1988learning} as the loss function with transitions from a near-optimal policy. Optimized using stochastic gradient descent, both methods used a batch size of 256 and a learning rate determined by grid search. Performance was evaluated by the RMSE between predictions and true values on unseen samples.~\looseness=-1

We consider two different arrangements of RBFs, a uniform layout and one that is denser towards the center of the environment. Note that although the true value function has a discontinuity spiraling out from the center, which might benefit from finer resolution, the more poorly conditioned nature of this non-uniform arrangement of features makes the problem harder, as can be seen in figure~\ref{fig:rl}. We see that FRM is competitive in the easier of the two cases while outperforming ERM by over 20\% in the harder one. We hypothesize that TD loss is commonly subject to complex noise that can hinder ERM when its features are poorly aligned. Furthermore, due to \textit{bootstrapping} ($\cL(s,r,s'):=(f_\theta(s)-r-\gamma f_\theta(s'))^2$) the temporal difference error is inherently functional through the term $f_\theta(s')$ affecting the label.

\subsection{FGM-based VAE finds better representations within structured variations}
To better understand when FRM works better than ERM, we build a Variational AutoEncoder(VAE) on top of MNIST~\citep{lecun1998gradient} and combinations of two popular variations: colored MNIST~\citep{arjovsky2019invariant} and translated MNIST~\citep{jaderberg2015spatial}. We build a vanilla VAE with MLP encoder and CNN decoder. Then, we evaluate the quality of the representation to do classification for the vanilla VAE and an FGM-based decoder where noise is modeled in function space. For FGM, we train a small MLP on top of the latent representation, with a stop-gradient, and measure accuracy depending on the size of the latent.

We see that in MNIST, where natural variations in orientation, translation, and color have been unnaturally removed, some gains exist but are small. In the datasets containing variations in color or translation, the FRM gains are substantial. This is because noise in CNN weights can easily explain these structured variations, as shown in figure~\ref{fig:structured}. Similarly, papers such as Deep Image Prior~\citep{ulyanov2018deep} have argued that neural networks are good models for real-world variability, making FRM particularly appealing for modeling real-world data. Results are shown in figure~\ref{fig:VAE_MNIST}.

\begin{figure}
    \centering
    \includegraphics[width=\linewidth]{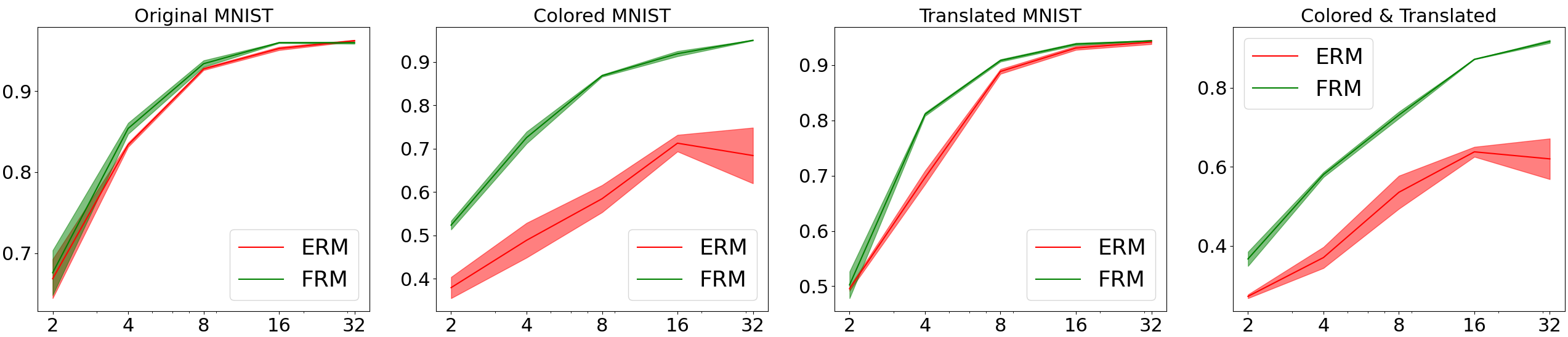}
    \caption{ Accuracies of an MLP trained from the latent space of two CNN-based VAEs, using ERM and FRM. Training with FRM improves performance over ERM in vanilla MNIST and all its variants. For the three variants, which have large, natural, structured variability, FRM provides large gains.\vspace{-5mm}}
    \label{fig:VAE_MNIST}
    \vspace{-1mm}
\end{figure}

\vspace{-3mm}
\section{Conclusion}
Machine Learning  has changed significantly since its inception: modern neural networks are often over-parameterized and datasets are now internet-scale. 
We propose FRM, an alternative framework to ERM, that is better suited for the current ML regime.
FRM proposes to measure losses in function space rather than output space. By improving on one of the foundations of ML, FRM can be applied to a wide range of settings. We show this through experiments in supervised, unsupervised, and reinforcement learning.

The main limitation of FRM in its general form is its compute cost. Thanks to the Taylor approximations proposed in sections~\ref{subsec:approximation} and~\ref{sec:experiments} we can already run FRM with mid-sized neural networks. 
This is because FRM requires evaluating a product between the Hessian-inverse and a Jacobian, which is an order of magnitude more expensive than a simple forward pass.
However, FRM could become orders of magnitude more efficient than ERM-based approaches. As explained in section~\ref{subsec:overparam}, under-parameterized FRM behaves similarly to over-parameterized ERM by making models have $n$ times more parameters $\theta_1,\dots,\theta_n$, one copy per training example. There, each $\theta_i$ is instantiated on the fly for loss computation and thus doesn't need to be in memory, which is crucial for modern datasets with millions of datapoints. Furtheremore, customizations to particular models such as the variational approximation for LLMs explored in section~\ref{subsec:variational} can retain the benefit of functional losses while remaining scalable.

In the last few years, there has been a clear tendency towards building large models capable of performing many tasks which were previously modeled using separate functions. FGMs propose the natural step to model the diversity in these datasets in function space rather than in output space. This enables richer and more meaningful noise models. Despite noise being pervasive across real-world data, modern deep learning approaches still use simple loss functions. As we also keep moving towards larger, more varied datasets, properly modeling the internal data diversity becomes crucial. We believe FRM opens a new research path towards an effective solution in this modern ML regime.
\newpage
\subsection*{Acknowledgments}
We would first like to thank Alexander Rakhlin for multiple insightful conversations on the work and for useful feedback. We would also like to thank Anurag Ajay, Peter Battaglia, Rohan Chitnis, Xavier Fern\'andez-Real, Javier Lopez-Contreras, Jiayuan Mao, Manon Revel, David Sontag, Oriol Vinyals, and Luca Zappella for their insightful comments and Maria Bauza for multiple conversations and providing great feedback on the draft.~\looseness=-1

\bibliography{main}
\bibliographystyle{tmlr}

\newpage
\appendix
\onecolumn

\appendix

\section{Proof of the Universal Distribution Theorem}~\label{app:universal}
\begin{theorem}[\textbf{Universal Distribution Theorem}]
Let   $l\ge 4$, $\Xcal=[0,1]^t$, $\Ycal=[0,1]^m$, and  $\Fcal_{\Theta}^{k,l}$ be a set of all functions represented by $l$-layer neural networks with sigmoidal activation and $k$ neurons per hidden layer. Let $q$ be a probability measure on  $(\Xcal \times \Ycal,  \Bcal(\Xcal \times \Ycal))$  such that $x\mapsto q(Y \in b_{y}|X=x)$ and  $\alpha\mapsto  q((1+\alpha)Y \in b_{y}|X=x)$ are continuous for $b_y \in \Bcal(\Ycal)$. Then, for any $\epsilon >0$, there exist $k \ge 1$ and $p \in \mathrm{FGM}[\Fcal_\Theta,\Xcal]$ such that $D_{\mathrm{TV}}[p, q] <\epsilon$, where $D_{\mathrm{TV}}$ is  the total variation distance.
\end{theorem}
\begin{proof}
  Let  
  $q$ be given and  $\epsilon>0$. For any $(b_{x},b_{y})\in  \Bcal(\Xcal \times \Ycal)$, since $p(X\in b_x,Y \in b_y) = p(X\in b_x)p(Y \in b_y|X\in b_x)$, we choose $ p(X\in b_x)=q(X\in b_x)$ for all $b_{x}$. With this, we have that  
\begin{align*}
D_{\mathrm{TV}}[p, q] &=\sup_{(b_{x},b_{y})\in  \Bcal(\Xcal \times \Ycal)}|p(X\in b_x,Y \in b_y)-q(X\in b_x,Y \in b_y)|
\\ & =\sup_{(b_{x},b_{y})\in  \Bcal(\Xcal \times \Ycal)} q(X\in b_x)\left|p(Y \in  b_{y}|X\in b_x)-q(Y \in  b_{y}|X\in b_x)\right|  
\\ & \le \sup_{b_{y}\in  \Bcal( \Ycal), x\in \Xcal} \left|p(Y \in  b_{y}|X=x)-q(Y \in  b_{y}|X=x)\right|.    
\end{align*}
Since $p \in \mathrm{FGM}[\Fcal_\Theta,\Xcal]$, we have that $p(Y \in  b_{y}|X=x)=\PP_{\theta \sim p(\theta)}[f_{\theta}(x)\in  b_{y}]$. Thus, 
$$
D_{\mathrm{TV}}[p, q] \le\sup_{b_{y}\in  \Bcal( \Ycal), x\in \Xcal} \left|\PP_{\theta \sim p(\theta)}[f_{\theta}(x)\in  b_{y}]-q(Y \in  b_{y}|X=x)\right|. 
$$ 
Thus, in the rest of the proof, we want to choose $f_{\theta}$ and $p(\theta)$ such that  $\left|\PP_{\theta \sim p(\theta)}[f_{\theta}(x)\in  b_{y}]-q(Y \in  b_{y}|X=x)\right|$ is arbitrarily small.
We achieve this in two steps. In the first step, we use the sign function (i.e., signum function) as activation function.
In the second step, we approximate the sign function with the sigmoid function and analyze  the effect of the approximation. 

We now start the first step (with the sign activation function). 

Define $c_{j}=j/r$ and $c'_{j}=(j+1)/r$ for $j=0,1,2\dots,r$ for some $r$. The first hidden layer uses deterministic weights and biases so that each neuron of the first layer implements the functions $\one\{x_i \ge c_{j}\}$ for each coordinate $i =1,2,\dots,n$ of $x$ and each threshold  $c_{j}$ with $j=0,1,2,\dots,r$. This uses $(r+1)n$ neurons in total. We use  additional  $(r+1)n$ neurons so that similarly, each neuron
 implements the functions $\one\{x_i <  c_{j}'\}$ for each coordinate $i $ and each threshold $j$. The first hidden layer uses $2(r+1)n$ neurons in total.

Each neuron of the second hidden layer is connected to the node of $\one\{x_i \ge c_{j}\}$ and the node of  $\one\{x_i <c_{j}'\}$ from the first hidden layer using deterministic weights and biases so that it computes $\bar \sigma(\one\{x_i \ge c_{j}\}+\one\{x_i <  c_{j}'\}-1.5)=\one\{x_i \in[ j/r ,(j+1)/r)\}$, where $\bar \sigma$ is the sign function in this first step of the proof. As a result, the second hidden layer uses  $(r+1)n$ neurons in total.

Similarly, 
using deterministic weights and biases, each neuron of the third hidden layer computes $\bar \sigma(\sum_{i=1}^t\one\{x_i \in[\varphi(i)/r ,(\varphi(i)+1)/r)\}-t+0.5)=\one\{x \in [\varphi(1)/r ,(\varphi(1)+1)/r) \times[\varphi(2)/r ,(\varphi(2)+1)/r) \times \cdots \times [\varphi(t)/r ,(\varphi(t)+1)/r) \}$, for each $\varphi \in \Phi$ where  $\Phi$  is the set of all functions from $\{1,2,\dots,t\} $ to $\{0,1,2\dots,r-1\}$ (i.e., we have $|\Phi|$ neurons at the third hidden layer). 

For the rest of hidden layers until $l-1$-th layer, we copy the output of the third hidden layer: i.e., $\bsigma(z-0.5)=\one\{z>0.5\}$ with the bias term being $0.5$.

By connecting the  $l-1$  layer to the output layer, we can write 
\begin{align}
f_{\theta}(x)&=\sum_{\varphi \in \Phi} w_{\varphi}\one\{x \in[\varphi(1)/r ,(\varphi(1)+1)/r) \times \cdots \times [\varphi(t)/r ,(\varphi(t)+1)/r)\}
\end{align}
where $w_{\varphi}\in \RR^{m}$ is the weight connecting the   third hidden layer neuron with $\varphi$ to the output layer. We define the distribution on this weight by $w_{\varphi} \sim q(\cdot|X=x_\varphi)$ for all $\phi \in \Phi$ and $y \in \Ycal$, where $x_\varphi=(\varphi(1)/r,\varphi(2)/r,\dots,\varphi(t)/r)$.  \ This implies that
  $\PP_{\theta \sim p(\theta)}[w_{\varphi}\in  b_{y}]=q(Y \in  b_{y}|X=x_\varphi)$ for all $\varphi \in \Phi$ and $b_{y}\in  \Bcal( \Ycal)$.  Since exactly one of those neurons at  the third hidden layer will be active for any given input $x$,  we have that $f_{\theta}(x)=w_{x}\in \RR^{m}$ where $w_x$ is the one of $w_{\varphi}$ such that $\one\{x \in[\varphi(1)/r ,(\varphi(1)+1)/r) \times \cdots \times [\varphi(t)/r ,(\varphi(t)+1)/r)\}=1$. Thus, we have that for any $\varphi \in \Phi$ and any $x \in[\varphi(1)/r ,(\varphi(1)+1)/r) \times \cdots \times [\varphi(t)/r ,(\varphi(t)+1)/r), $
$$
\PP_{\theta \sim p(\theta)}[f_{\theta}(x)\in  b_{y}]=\PP_{\theta \sim p(\theta)}[w_{\varphi}\in  b_{y}]=q(Y \in  b_{y}|X=x_\varphi)=q(Y \in  b_{y}|X=x)+\delta_1,
$$
where  $\delta_1 \rightarrow 0$ as $r \rightarrow \infty$ from the continuity of the function $x\mapsto q(Y \in B_{\epsilon} (y)|X=x)$. Therefore, for all $x \in \Xcal$ and  $b_{y}\in  \Bcal( \Ycal)$, 
\begin{align}
\PP_{\theta \sim p(\theta)}[f_{\theta}(x)\in  b_{y}] =q(Y \in  b_{y}|X=x)+\delta_1,
\end{align} 
where  $\delta_1 \rightarrow 0$ as the size of the network increases ($r \rightarrow \infty$). This completes the first step.

We now consider the second step where we approximate the sign activation with the sigmoid activation. For the first three hidden layers, we use the same construction as the first step above except that we use   the arbitrarily large scale  of weights and biases  to approximate the sign function. With this, the output of the the  $l-1$  layer  becomes similar to that of the sign activation but with the error  term $\epsilon_2$:
\begin{align}
f_{\theta}(x)&=\sum_{\varphi \in \Phi} w_{\varphi}(\one\{x \in[\varphi(1)/r ,(\varphi(1)+1)/r) \times \cdots \times [\varphi(t)/r ,(\varphi(t)+1)/r)\}+\epsilon_2)
\end{align}
where $\epsilon_2 \rightarrow 0$ as the scales of weights and biases within first three hidden layers approach infinity. For the next year, we use the exactly  the same construction. As a result,
for any $\varphi \in \Phi$ and any $x \in[\varphi(1)/r ,(\varphi(1)+1)/r) \times \cdots \times [\varphi(t)/r ,(\varphi(t)+1)/r)$,
\begin{align*}
\PP_{\theta \sim p(\theta)}[f_{\theta}(x)\in  b_{y}]&=\PP_{\theta \sim p(\theta)}[(1+\epsilon_2)w_{\varphi}\in  b_{y}] 
\\ & =q((1+\epsilon_2)Y \in  b_{y}|X=x_\varphi)
\\ &=q(Y \in  b_{y}|X=x_\varphi)+ \delta_2
\end{align*}
where  $\delta_2 \rightarrow 0$ as  $\epsilon_2 \rightarrow 0$  based on the continuity of the function  $\alpha\mapsto  q((1+\alpha)Y \in b_{y}|X=x)$. Moreover, $q(Y \in  b_{y}|X=x_\varphi)+ \delta_2=q(Y \in  b_{y}|X=x)+\delta_1+\delta_2$ from the continuity  of the function $x\mapsto q(Y \in B_{\epsilon} (y)|X=x)$ as before. 
Thus, for all $x \in \Xcal$ and  $b_{y}\in  \Bcal( \Ycal)$, 
\begin{align}
\PP_{\theta \sim p(\theta)}[f_{\theta}(x)\in  b_{y}]=q(Y \in  b_{y}|X=x)+\delta_1+ \delta_2.
\end{align} 
By combining the upper bound on  $D_{\mathrm{TV}}$ above, we have that 
$$
D_{\mathrm{TV}}(p,q) \le|\delta_1+ \delta_2|  ,  
$$
where  $\delta_1 \rightarrow 0$ as the size of the network increases ($r \rightarrow \infty$), and   $\delta_2 \rightarrow 0$ as  the scales of weights and biases (within first three hidden layers) increases.  

\end{proof}

\section{Proofs of empirical losses being sub-cases of functional losses}~\label{app:equivalent}
\subsection{Mean-squared error and $L_1$ loss as a functional losses}~\label{subsubsec:app_MSE}
Let our dataset $\dtrain=\{(x_i,y_i)\}_{i=1}^n$, $y_i\in\bR^1$, and let $\cL_{MSE} = \frac{1}{n}\sum_{i=1}^n (f(x_i)-y_i)^2$, i.e. the mean-squared error loss. 

\begin{lemma}
For any arbitrary function class $f_{\theta,\beta}(x)$ expressible as $f_{\theta,\beta}(x) = f_{\theta}(x)+\beta$, there exists a functional loss restricted to functional adaptations $\theta_i=\theta$ that only change $\beta\rightarrow\beta_i$ which is equivalent to the mean-squared error loss.
\end{lemma}
\textbf{Proof} Since we can only change $\beta$ there is a single solution to the per-point constraint: $f_\theta(x_i) = f_{\theta}(x_i)+\beta_i = y_i \Rightarrow \beta_i = y_i-f_{\theta}(x_i)$.
We can now model the probability distribution over functions $\cF\left(\theta,\beta_i|\theta,\beta,\cL_{MSE}\right)$ as a gaussian centered at $(\theta,\beta)$. Since $\theta$ doesn't change, this will just be $\cN(\beta_i-\beta)$. Maximizing the mean of the log-probabilities will result in $\frac{1}{n}\sum_i \log\cN(\beta_i-\beta) = \frac{1}{n}\sum_i(\beta_i-\beta)^2 = \frac{1}{n}\sum_i(y_i-f_{\theta}(x_i)-\beta)^2 = \frac{1}{n}\sum_i(y_i-f_{\theta,\beta}(x_i))^2=\cL_{MSE}$.

Of note, the Gaussian model of the functional distribution satisfies $$\cF\left(\theta,\beta_i|\theta,\beta,\cL_{MSE}\right)=\cN\left((\theta,\beta_i)-(\theta,\beta)\right) \propto e^{-|\beta-\beta_i|^2} = e^{-\bE_x\cL_{MSE}(f_{\theta,\beta},f_{\theta',\beta'})}.$$
This is because for all $x$, $\cL_{MSE}\left(f_{\theta,\beta}(x)-f_{\theta',\beta'}(x)\right) = |f_{\theta,\beta}(x)-f_{\theta',\beta'}(x)|^2 = |\beta-\beta'|^2$.

Finally, we note that the entire derivation can be equivalently followed for the L1 loss by swapping $|\cdot|^2$ for $|\cdot|$ and the Gaussian distribution for the Laplace distribution.
\subsection{Classification error as a functional loss}~\label{subsubsec:app_class}
Let us now look at multi-class classification and let our dataset $\dtrain=\{(x_i,y_i)\}_{i=1}^n$, $y_i\in\{1,\dots,C\}$. Our function class will output in an unconstrained \textit{logit} space $\bR^C$ and we define $\cL_{cls} = \frac{1}{n}\sum_{i=1}^n \bOne\llbracket y_i=\argmax{c} \left(f_{\theta,\beta}(x_i)\right)_c\rrbracket$, i.e. the classification error. As in previous sections, abusing notation we will refer to $\bOne\llbracket y_i=\argmax{c} \left(f_{\theta,\beta}(x_i)\right)_c\rrbracket$ as $\bOne\llbracket y_i=f_{\theta,\beta}(x_i)\rrbracket$.

\begin{lemma}
For any arbitrary function class $f_{\theta,\beta}(x)$ expressible as $f_{\theta,\beta}(x) = f_{\theta}(x)+\beta$, $\beta\in\bR^c$, constrained on $f_\theta(x)$ being finite, there exists a functional loss restricted to functional adaptations $\theta_i=\theta$ that only change $\beta\rightarrow\beta_i$ which is equivalent to the classification error.
\end{lemma}

\textbf{Proof} We will show that a solution is given by $\cF\left(\theta,\beta_i|\theta,\beta,\cL_{cls}\right) = p\cdot\delta(\beta_i-\beta)+(1-p)\lim_{\sigma\rightarrow\infty}\cN(0,\sigma)(\beta)$, with $p=\frac{e-1}{C+e-1}\in(0,1)$. In other words, a specific positive (note the open brackets) combination of an infinitely-sharp distribution (Dirac's delta) with an infinitely-flat distribution. Given a fixed $p,\theta,\beta$, the probability of $y_i=\argmax{c} f_{\theta_i,\beta_i}(x_i)$ will be equal to $p\cdot\left[y_i=\argmax{c} \left(f_{\theta,\beta}\right)_c\right] + \frac{1-p}{C}$. This comes directly from the definition of the functional probability distribution: with probability $p$, we have $(\theta_i,\beta_i) = (\theta,\beta)$ and thus the result depends solely on $(\theta,\beta)$; with probability $(1-p)$ the logits are perturbed by an infinitely strong noise and thus the $\argmax{}$ will just be a uniform distribution over the classes, i.e. $\frac{1}{C}$.

Now, the average log-likelihood of the functional loss will be:
\begin{align*}
    \frac{1}{n}\sum_{i=1}^n \log\left(p\cdot\bOne\llbracket y_i=f_{\theta,\beta}(x_i)\rrbracket + \frac{1-p}{C}\right) &= \\
    \log{\frac{1-p}{C}} + \frac{1}{n}\sum_{i=1}^n \log\left(\frac{p\cdot\bOne\llbracket y_i=f_{\theta,\beta}(x_i)\rrbracket +(1-p)/C}{(1-p)/C}\right) &= \\
    \log{\frac{1-p}{C}} + \log\left(\frac{p+(1-p)/C}{(1-p)/C}\right)\frac{1}{n}\sum_{i=1}^n \bOne\llbracket y_i=f_{\theta,\beta}(x_i)\rrbracket &= \\
    \log{\frac{1-p}{C}} + \log\left(1+\frac{pC}{1-p}\right)\cL_{cls}&=\\
    -\log\left(C+e-1\right)+\cL_{cls}&.
\end{align*}
where in the second step we observe that the log term within the sum is zero when $y_i\neq f_{\theta,\beta}(x_i)$ and, in the last step, we have set $p=\frac{e-1}{C+e-1}$, which by construction is in $(0,1)$. We can now easily see that this is equivalent to $\cL_{cls}$ up to a constant additive term, which will not affect any optimization.
\subsection{Cross-entropy loss as a functional loss}~\label{subsubsec:app_CE}
Continuing in multi-class classification and let our dataset $\dtrain=\{(x_i,y_i)\}_{i=1}^n$, $y_i\in\{1,\dots,C\}$. Our function class will output in an unconstrained \textit{logit} space $\bR^C$ and we define $\cL_{CE} = \frac{1}{n}\sum_{i=1}^n \log \sigma\left(f_{\theta,\beta}\right)_{y_i}$, i.e. the cross-entropy loss. Here, $\sigma(\cdot)_c$ corresponds to taking the $c$-th component of the softmax of a given logit to obtain the probability of a given class $c$ given the logit predictions.

\begin{lemma}
For any arbitrary function class $f_{\theta,\beta}(x)$ expressible as $f_{\theta,\beta}(x) = f_{\theta}(x)+\beta$, $\beta\in\bR^C$, there exists a functional loss restricted to functional adaptations $\theta_i=\theta$ that only change $\beta\rightarrow\beta_i$ which is equivalent to the cross-entropy loss.
\end{lemma}

\textbf{Proof} As shown in~\citep{jang2016categorical,maddison2016concrete} if we have logits $\gamma_c=f_\theta(x_i)_c+\beta_c$ we can sample from the probability distribution of distribution equal to $\sigma(\gamma)$ by $c=\argmax{i}(\gamma_i+g_i$) where each $g_i$ follows an independent Gumbel distribution, i.e. $g_i=-\log(-\log u_i), u_i\sim \cU(0,1)$. This gives us a trivial expression for a functional distribution over which to make maximum likelihood: $\beta_i\sim \beta + \cG$, where $\cG$ consists of $c$ independent Gumbel noise variables. This is because, since $\beta$ lives in logit space, adding noise to $\beta$ is equivalent to adding noise to the logits. Finally, since the cross-entropy loss is the maximum likelihood assuming a probability distribution given by the logits and we have shown a functional distribution with the same distribution, performing maximum likelihood on that distribution is equivalent to minimizing the cross-entropy loss.

\section{Example: predicting house prices with linear regression} ~\label{sec:house_prices}
Let's consider predicting the price of a house given its surface area using a linear regressor: $y=\lambda x + \beta$ and the mean-squared error loss function. ERM would simply find the $\lambda,\beta$ leading to the lowest squared error on the training data. This is equivalent to doing maximum likelihood on a gaussian noise model $y_i \sim N(\lambda x_i + \beta,\sigma^2)$ with constant $\sigma$. 
However, this may be suboptimal. For instance, we intuitively know that prices of bigger houses tend to be higher, but also have larger variance: we expect the price of a very large house to vary by 500k, but we would not expect the same 500k variation for a small house.

When using the FRM framework, we assume that, for each house $(x_i, y_i)$ there are different $\lambda_i,\beta_i$, satisfying $y_i = \lambda_i x_i + \beta_i$. For instance, we may believe that agent commissions vary and are well-modeled by $\beta_i$, and that the price-per-meter-squared (captured by $\lambda_i$) changes depending on the neighborhood. This is the modeling made by FGMs, which is more flexible than the output-level noise model corresponding to mean-squared error. We show this effect in figure~\ref{fig:HoughTransform}.

\section{Functional noise in a CNN}
To show the value of the Taylor approximation, we create a dataset by sampling different parameter assignments on a 4-layer CNN architecture. The CNN takes in a CIFAR-10 image and outputs a real number. We provide only 8 labels to each method, allowing empirical risk minimization to easily memorize the dataset. Despite FRM obtaining substantially higher training losses ($.000$ vs $.052$), we observe FRM obtains significantly less test error ($.125$ to $.085$).

We also test the ability of FRM to modify its training depending on the loss function. Although this is obviously the case for ERM, in approximate FRM the loss function enters only in an indirect way, affecting the hessian in equation~\ref{eq:Taylor_FRM_objective}. We modify the objective by creating two different losses, which assign zero loss to labels that are either positive or negative, respectively. Table~\ref{tab:toy_CNN} shows that indeed FRM performs better when trained and tested on the same loss ($0.085$ vs $0.128$).

\begin{table}[h]
    \centering
    \caption[FRM outperforms ERM in a small CNN environment]{FRM outperforms ERM in a small CNN environment despite ERM having 0 training loss. Furthermore, the Hessian can be enough to express the dependence on the loss function.}
    \begin{tabular}{@{}llcccc@{}}
        \toprule
        \multicolumn{2}{c}{Objective}& \multicolumn{2}{c}{Train} & \multicolumn{2}{c}{Test} \\ \cmidrule(lr){3-4} \cmidrule(lr){5-6}
         && positives & negatives & positives & negatives\\
        \midrule
         ERM & positives & \textbf{.000} $\pm$ \textbf{.000}& $.283\pm .016$ & $.130\pm.006$& $.278\pm.013$\\
        & negatives & $.336\pm .020$& \textbf{.000} $\pm$ \textbf{.000} & $.323\pm.018$& $.119\pm.005$\\  \hline
        FRM &positives & $.052\pm .002$ & $.109\pm .007$ & \textbf{.085} $\pm$ \textbf{.004} & $.124\pm .006$ \\
        & negatives & $.131\pm .010$ & $.052\pm .002$ & $.136\pm .009$ & \textbf{.084}$\pm$\textbf{.004} \\
        \bottomrule
    \end{tabular}
    \label{tab:toy_CNN}
\end{table}

\section{Further understanding the difference between ERM and FRM}
\textbf{The ERM assumption: } by assuming that the training objective is equal to the test loss $\cL$, ERM can be suboptimal for certain $\cP(\theta)$, like the house example on section~\ref{subsec:sampling_functions}. As shown in appendix~\ref{app:equivalent}, for many loss functions $\cL$, including most of the common ones, ERM is equivalent to assuming the functional generative model and then doing maximum likelihood on $\cP(\theta)$ by assuming it has a form parameterized by $\hat{\theta}$ whose uncertainty is only on the output offset parameters. In other words, the assumption equivalent to performing ERM is often strictly more assuming than FGMs.

For instance, consider predicting the price of different houses as a function of their size and having MSE as the loss. Doing empirical risk minimization with the MSE would be equivalent to doing maximum likelihood on the following price model: $y_i\sim\cN(f(x_i),\sigma^2)$. However, we would expect noise to be \textit{heteroskedastic} with higher variations for higher prices.
Thus, even if we are evaluated on MSE on the test data, it may not be advisable to use it as our training criteria.

Similarly, consider a child learning a concept from examples on a textbook rather than from standardized images of a dataset. Images may receive different illuminations from the sunlight, or be in different positions than we expect. These factors will produce massive changes in pixel space, but in very structured ways~(fig~\ref{fig:structured}). However, humans can still easily grasp the idea because the 'conceptual' noise is small. 

How can we have more meaningful noise models? By construction, we will often believe that the function class $f_\theta$ is a good characterization of the relationship between $x$ and $y$. It is thus a natural assumption to define a noise model by leveraging the function class itself. More concretely, we can think of a generative model of the data as first sampling the input $x_i$, then sampling a function $f_i\sim\cF(\mathcal{L},\theta)$ from some parameterized distribution over functions, which will depend on both the problem-specified loss function $\cL$ as well as the function class $f_\theta$. Once the function and the input have been sampled, the output is automatically determined $y_i=f_i(x_i)$, see the right of figure~\ref{fig:generative_model}.

For example, in our house-price prediction, if we are using a linear model $f(x)=\lambda \cdot x + \beta$, then it makes sense to think about our data as coming from first sampling $x_i\sim p(x)$ and $(\lambda_i,\beta_i)\sim\cF(\cL, (\lambda,\beta))$, then computing $y_i=\lambda_i\cdot x_i+\beta_i$, as shown on the right of figure~\ref{fig:generative_model}. For instance, $\beta_i$ can model different commissions or taxes, and $\lambda_i$ can model the per-meter-square price being variable across neighborhoods. Even if we care about making accurate predictions in dollar-space, assuming our uncertainty is only in the offset term $\beta_i$ may be too restrictive.

\hspace{-3mm}
\subsection{ERM vs FRM for the linear case}~\label{subsec:vislin}
Let us now take a deeper look at our linear regression example. We have a dataset $\dtrain=\{(x_i,y_i)\}$, depicted in the top-right of figure~\ref{fig:HoughTransformA}, with an arbitrary color per point. For every point, there is a subspace of models $(\lambda_i,\beta_i)$ s.t. $\lambda_ix_i+\beta_i=y_i$. Since we only have two parameters, we can also look at function-space in 2-D, and plot the corresponding subspace for each point, in the bottom-left of figure~\ref{fig:HoughTransformA}. We observe that every point gives us a line in function space, which we plot with the corresponding color.

Our goal is to produce a probability distribution $\cP(\lambda,\beta)$ such that the sum of the log-densities of each line $(\lambda_i,\beta_i)_{\lambda_ix_i+\beta_i=y_i}$ is maximal. Intuitively, this means that each line should pass through a high-density area of the probability distribution, but it does not mean that the line should be covered by the high-density area (which is not possible, since they're unbounded). This can be seen in figure~\ref{fig:HoughTransformB} where all lines pass near the center of the distribution generating the data (marked in green).

We can further see that ERM with the MSE loss is equivalent to finding a point $(\lambda^{ERM},\beta^{ERM})$ that minimizes the \textit{vertical} distance to each line: 

\begin{align*}
(\lambda^{ERM},\beta^{ERM}) = &\min_{\lambda,\beta} \sum_i (y_i-\lambda x_i-\beta)^2 \\  =&\min_{\substack{\lambda,\beta,\\\lambda^i:\lambda^i=\lambda}} \sum_i (y_i-\lambda_ix_i-\beta)^2\\ =&\min_{\substack{\lambda,\beta,\{\lambda^i,\beta^i\}:\\\lambda^i=\lambda,\\\lambda_ix_i+\beta_i=y_i,}} \sum_i (\beta_i-\beta)^2.    
\end{align*}

\vspace{-3mm}
In contrast, if the probability distribution in parameter space is a Gaussian, FRM involves taking the distance of the entire vector $(\lambda,\beta)$, using the inverse covariance matrix as the metric. For cases where most of the uncertainty is in the slope, as in figure~\ref{fig:HoughTransformB}, ERM measures the distance in the vertical direction and FRM measures it almost horizontally, leading to different results.

\subsection{Visualization for a simple fully convolutional network}~\label{subsec:visconv}
Figure~\ref{fig:structured} shows the difference between MSE and its functional correspondent for a small fully-convolutional network mapping images to images $f_\theta$. Images $y$ with the same empirical loss $|y-f_\theta(x)|^2$ could require very different functional adaptations to explain: $\min_{\theta':f_{\theta'}(x)=y}|\theta'-\theta|_{f,\cL}$. For instance, if one does edge detection and mistakenly translates its prediction a bit to the right, this small change in functional space could lead to a large error in pixel space. Similarly, if we have a pattern detector and we slightly change its threshold, it could make the entire prediction darker or lighter.~\looseness=-1

Conversely, if we add unstructured noise onto our image, it is to be expected that it will have a high functional loss as no small perturbation of the function could simultaneously explain pure noise. That's indeed what we observe in figure~\ref{fig:structured} when we look for images with high and low functional loss for a fixed empirical loss. Images with high functional loss contain salt-and-pepper-like noise that breaks the smooth pattern of the original image. In contrast, images with low functional loss preserve the overall structure while uniformly shifting large blocks of pixels to a much lighter color. If the noise in our data is better represented by our functional class than noise in the output, we can take this into account to improve learning.

\end{document}